\documentclass[journal]{IEEEtran}

\usepackage{xcolor,soul,framed} %,caption

\colorlet{shadecolor}{yellow}
\usepackage[pdftex]{graphicx}
\graphicspath{{../pdf/}{../jpeg/}}
\DeclareGraphicsExtensions{.pdf,.jpeg,.png}

\usepackage[font=small]{caption}
\usepackage[cmex10]{amsmath}
\usepackage{array}
\usepackage{mdwmath}
\usepackage{multirow}
\usepackage{booktabs, siunitx}
\usepackage{mdwtab}
\usepackage{eqparbox}
\usepackage{url}
\usepackage{amsfonts}
\usepackage{amsthm}
\newtheorem{theorem}{Theorem}

\usepackage{amssymb}
\usepackage{pifont}
\newcommand{\cmark}{\ding{51}}%
\newcommand{\xmark}{\ding{55}}%

\usepackage{ulem}  
\usepackage{algorithm}
\usepackage{algpseudocode}
\usepackage{subcaption}
\usepackage{longtable}
\usepackage{multicol} 
\usepackage{xcolor}
\usepackage{tablefootnote}
\usepackage{todonotes}
\usepackage{float}
\usepackage{adjustbox}
\usepackage{tabularx} % Include this in your preamble
\usepackage{hyperref}
\usepackage{dirtytalk}
\usepackage{morewrites}
\begin{document}
% \bstctlcite{IEEEexample:BSTcontrol}
    \title{MCPST: A Multi-Phase Consensus and Spatio-Temporal Learning for Few-Shot Traffic Forecasting via Physical Dynamics}
    
  \author{Abdul Joseph Fofanah, ~\IEEEmembership{Member,~IEEE,}
        Lian Wen,~\IEEEmembership{Member,~IEEE,}
        David Chen,~\IEEEmembership{Member,~IEEE}

 \thanks{This work was supported in part by Griffith University under Grant 58455.}
 
\thanks{Abdul Joseph Fofanah, Lian Wen, and David Chen are with the School of Information and Communication Technology, Griffith University, Brisbane, 4111, Australia. (e-mail: abdul.fofanah, l.wen, david.chen, orcid: 0000-0001-8742-9325; 0000-0002-2840-6884; 0000-0001-8690-7196)}

\thanks{\textit{Corresponding Author:}  abdul.fofanah@griffithuni.edu.au}
}

% The paper headers
\markboth{IEEE Transactions on Internet of Things, January~2026}{Abdul Joseph \MakeLowercase{\textit{Fofanah et al.}}:PIMCST: Physics-Informed Multi-Phase Consensus and Spatio-Temporal  Few-Shot Learning for Traffic Flow Forecasting}

% ====================================================================
\maketitle

\begin{abstract}
 Accurate traffic flow prediction remains a fundamental challenge in intelligent transportation systems, particularly in cross-domain, data-scarce scenarios where limited historical data hinders model training and generalisation. The complex spatio-temporal dependencies and nonlinear dynamics of urban mobility networks further complicate few-shot learning across different cities. This paper proposes MCPST, a novel Multi-phase Consensus Spatio-Temporal framework for few-shot traffic forecasting that reconceptualises traffic prediction as a multi-phase consensus learning problem. Our framework introduces three core innovations: (1) a multi-phase engine that models traffic dynamics through diffusion, synchronisation, and spectral embeddings for comprehensive dynamic characterisation; (2) an adaptive consensus mechanism that dynamically fuses phase-specific predictions while enforcing consistency; and (3) a structured meta-learning strategy for rapid adaptation to new cities with minimal data. We establish extensive theoretical guarantees, including representation theorems with bounded approximation errors and generalisation bounds for few-shot adaptation. Through experiments on four real-world datasets, MCPST outperforms fourteen state-of-the-art methods in spatio-temporal graph learning methods, dynamic graph transfer learning methods, prompt-based spatio-temporal prediction methods and  cross-domain few-shot settings, improving prediction accuracy while reducing required training data and providing interpretable insights. The implementation code is available at \url{https://github.com/afofanah/MCPST}.

\end{abstract}


\IEEEpeerreviewmaketitle

\section{Introduction}
\label{sec:introduction}

\IEEEPARstart{A}{ccurate} traffic flow forecasting is a cornerstone of intelligent transportation systems, enabling proactive congestion management, efficient route optimisation, and sustainable urban mobility \cite{li2017diffusion, wu2020connecting}. Despite significant advances in spatio-temporal graph neural networks, accurate prediction remains challenging in data-scarce scenarios and across diverse urban environments due to nonlinear dynamics, complex spatial dependencies, and multi-scale temporal variations \cite{zhang2021traffic, jiang2023spatio}.

Existing traffic prediction approaches predominantly rely on data-intensive graph neural networks and temporal sequence models that require extensive historical data \cite{carianni2025overview, zheng2024exploring}. These methods often assume stationary patterns and overlook the multi-modal nature of traffic dynamics, which involves concurrent processes such as congestion propagation, rhythmic flow patterns, and structural network influences \cite{jia2020predicting, ji2023spatio}. More critically, they lack principled integration of complementary dynamic principles: diffusion for propagation, synchronisation for rhythmic patterns, and spectral analysis for structural influence, that could provide the strong inductive biases necessary for robust few-shot learning and cross-domain generalisation. Standard message-passing schemes also fail to leverage intrinsic consensus and multi-phase regularities that could improve generalisation with limited data \cite{rodrigues2016kuramoto, shabaz2025ai}.

The challenge of few-shot traffic prediction stems from several fundamental limitations: spatio-temporal models require extensive city-specific data to learn complex dependencies \cite{fofanah2025chamformer, li2024physics}; traditional approaches lack mechanisms for integrating multi-phase dynamics; and fixed architectures cannot adapt to varying traffic regimes across urban layouts with limited samples \cite{liu2024spatial, finn2017model}. This problem becomes particularly critical when traffic sensors are newly deployed, historical data is incomplete, or systems must rapidly adapt to unprecedented events \cite{deng2019exploring}.

Current methods exhibit fragmented modelling approaches that specialise in spatial dependencies through graph networks \cite{wu2020connecting}, temporal patterns through sequence models \cite{fofanah2025chamformer}, or statistical regularities through time series analysis \cite{li2017diffusion}. While some works incorporate transfer learning \cite{Hu2024PromptBasedSG} and uncertainty quantification \cite{qian2023towards}, they lack a unified framework that captures the multi-phase consensus dynamics inherent in traffic systems, limiting their ability to generalise under few-shot conditions.

 This paper introduces \textbf{MCPST}, effecting a paradigm shift in traffic forecasting by developing the first unified framework that models traffic as a multi-phase consensus system. We reconceptualise traffic prediction not merely as a spatio-temporal pattern recognition task, but as a consensus learning problem where distinct yet complementary dynamic processes of diffusion, synchronisation, and structural propagation that interact and are fused to achieve robust predictions with limited data. Our work establishes a new foundation for few-shot learning in spatio-temporal systems by demonstrating that:
\begin{enumerate}
    \item \textit{Traffic is inherently a multi-phase system}, where dynamics emerge from interactions that exhibit regularities across different urban environments, enabling cross-domain generalisation through consensus-based transfer rather than data-based interpolation.
    
    \item \textit{Modelling requires multi-phase consensus}: No single modelling perspective captures the full complexity of traffic flow, which requires simultaneous integration of diffusion, synchronisation, and structural propagation mechanisms.
    
    \item \textit{Consensus as an inductive bias}: The multi-phase framework provides powerful, physics-inspired regularities that dramatically reduce sample complexity, making accurate prediction possible with orders of magnitude less data than purely monolithic approaches.
\end{enumerate}

This work represents a foundational departure from incremental architectural improvements by introducing a complete multi-phase consensus theoretical framework for traffic prediction with provable guarantees.

We instantiate this paradigm through \textit{MCPST (Multi-Phase Consensus and Spatio-Temporal learning)}, a novel framework where neural networks learn to discover and fuse underlying multi-phase dynamics. MCPST integrates three core dynamic modules: \textit{diffusion} (for congestion propagation), \textit{synchronisation} (for traffic rhythms), and \textit{spectral-structural analysis} (for network influence) that is within a unified encoder. These components are fused via an adaptive multi-phase consensus mechanism that dynamically weights their contributions based on reliability, enabling robust, data-efficient prediction across diverse urban layouts with severely limited training data.

Our work makes the following contributions:

\begin{itemize}
    \item We propose \textit{MCPST: a Multi-Phase Consensus and Spatio-Temporal framework} for few-shot traffic forecasting, unifying three complementary phases: diffusion-based propagation, synchronisation-based rhythm modelling, and spectral-structural influence analysis within a novel architecture that enables comprehensive few-shot learning through multi-phase consensus.
    
    \item We develop an adaptive multi-phase consensus fusion mechanism with reliability-aware attention weighting and consistency regularisation, along with horizon-specific prediction heads featuring built-in uncertainty quantification. This enables robust adaptation across diverse urban layouts by dynamically balancing phase contributions based on estimated reliability.
    
    \item We establish comprehensive theoretical guarantees for multi-phase consensus traffic prediction, including three novel representation theorems (Theorems \ref{thm:diffusion_representation}, \ref{thm:sync_representation}, \ref{thm:spectral_representation}) with bounded approximation errors for each dynamic phase. We further prove unified error bounds for the integrated framework and generalisation guarantees for few-shot adaptation (Theorem \ref{thm:consensus_fusion}).
    
    \item Through extensive experiments on four benchmark datasets, we demonstrate that MCPST outperforms fourteen state-of-the-art methods in few-shot cross-domain settings, achieving significant improvements in prediction accuracy while reducing required training data by up to 90\%. Our framework provides interpretable multi-phase analysis and establishes a new state-of-the-art for data-efficient traffic prediction.
\end{itemize}

The remainder of the paper is structured as follows. Section~\ref{sec:preliminary} introduces preliminaries and the few-shot cross-domain traffic prediction problem formulation. Section~\ref{sec:method} presents our MCPST methodology with detailed descriptions of each dynamic phase and the meta-learning framework. Section~\ref{sec:theoretical_analysis} provides theoretical analysis of convergence, stability, and generalisation guarantees. Section~\ref{sec:experiments} reports experimental results and analysis under cross-domain few-shot settings. Section~\ref{sec:conclusion} concludes with findings and future directions.
\section{Related Works}
\label{sec:related_works}

\subsection{Multi-Phase and Structured Learning}
Structured learning approaches have become increasingly important for capturing the inherent dynamics in traffic systems. In traffic forecasting, methods have incorporated structured models based on diffusion processes \cite{chen2024traffic}, neural networks with differential constraints \cite{li2025embedding}, and temporal pattern analysis through dynamic systems \cite{huang2019diffusion, ji2022stden}. Recent works have explored diffusion principles for graph smoothing \cite{nie2025predicting} and synchronisation dynamics for temporal modelling \cite{rodrigues2016kuramoto}, while spectral graph theory has been applied to capture structural properties \cite{chung1997spectral, lv2023ts, wang2022spectral}. However, these approaches typically operate within isolated modelling paradigms and lack integration with comprehensive spatio-temporal frameworks, particularly in few-shot scenarios where structured consistency could provide crucial inductive biases for generalisation across domains with limited data.

\subsection{Few-Shot and Cross-Domain Adaptation}
Few-shot learning methods have gained traction for addressing data scarcity in traffic prediction, with meta-learning approaches like MAML \cite{finn2017model} and graph meta-learning \cite{feng2024federated, fang2020meta} enabling rapid adaptation to new urban environments. Cross-domain transfer techniques have employed adversarial training \cite{zhang2021diverse}, graph structure matching \cite{peng2021dynamic}, and domain alignment \cite{li2024physics} to bridge distribution gaps between cities. Uncertainty quantification methods utilising Bayesian networks \cite{sun2006bayesian}, Monte Carlo dropout \cite{qian2023towards}, and deep ensembles \cite{zhan2018consensus} provide reliability estimates crucial for few-shot scenarios. However, current approaches lack integration of multi-phase principles as consistent constraints during adaptation, missing opportunities to leverage structured dynamic patterns for more robust cross-domain generalisation.

\section{Preliminaries}
\label{sec:preliminary}

We formalise the multi-phase few-shot traffic forecasting problem with the following definitions that establish the mathematical foundation for our MCPST framework.

\noindent \textit{Definition 1: Spatio-Temporal Traffic Network.} A traffic network at time \(t\) is represented as \(\mathcal{G}^{(t)} = (\mathcal{V}, \mathcal{E}, \mathbf{X}^{(t)})\), where \(\mathcal{V} = \{v_1, \dots, v_N\}\) is a set of \(N\) traffic sensor nodes, \(\mathcal{E} \subseteq \mathcal{V} \times \mathcal{V}\) denotes edges representing spatial dependencies (e.g., road connectivity), and \(\mathbf{X}^{(t)} \in \mathbb{R}^{N \times F}\) is the node feature matrix containing traffic measurements (flow, speed, occupancy). The historical spatio-temporal input is a sequence \(\mathbf{X}_{[t-T+1:t]} \in \mathbb{R}^{T \times N \times F}\) over \(T\) time steps.

\noindent \textit{Definition 2: Multi-Phase Traffic Features.} The representation comprises three complementary feature sets derived from distinct dynamic paradigms:

\begin{itemize}
    \item \textit{Diffusion Features} \(\mathbf{F}_{\text{diff}} \in \mathbb{R}^{N \times D_d}\): Encoded from traffic propagation dynamics, capturing congestion wave formation and flow distribution through learnable propagation parameters.

    \item \textit{Synchronisation Features} \(\mathbf{F}_{\text{sync}} \in \mathbb{R}^{N \times D_s}\): Derived from coupled dynamics, representing traffic rhythm patterns with phase variables \(\phi_k \in [0, 2\pi)\), intrinsic frequencies \(\nu_k\), and coupling strengths \(\gamma_k\).

    \item \textit{Spectral Structural Features} \(\mathbf{F}_{\text{spec}} \in \mathbb{R}^{N \times D_p}\): Obtained from graph Laplacian eigendecomposition \(\mathbf{L}\mathbf{\Psi} = \mathbf{\Psi}\mathbf{\Lambda}\), capturing network topology properties through eigenvectors \(\mathbf{\Psi}\) and eigenvalues \(\mathbf{\Lambda}\), with spectral gap \(g = \lambda_2 - \lambda_1\) indicating connectivity.
\end{itemize}

\noindent \textit{Definition 3: Multi-Phase Few-Shot Forecasting Task.} Given a target traffic network \(\mathcal{T}\) with only \(K\) labelled samples (small \(K\), e.g., 1-10 time periods), learn a prediction function \(f_\theta: \mathbf{X}_{[t-T+1:t]}^{(\mathcal{T})} \rightarrow \mathbf{Y}_{[t+1:t+H]}^{(\mathcal{T})}\) by transferring knowledge from source networks with abundant data while adapting to \(\mathcal{T}\)'s unique spatio-temporal patterns through multi-phase consensus learning.

\noindent \textit{Definition 4: Multi-Phase Conditioned Meta-Learning Episode.} Adopting episodic meta-learning, each episode samples a support set \(\mathcal{K}^{\text{supp}} = \{(\mathbf{X}_k, \mathbf{Y}_k, \mathbf{F}_{\text{multi}}^{(k)})\}_{k=1}^{K}\) and a query set \(\mathcal{K}^{\text{query}}\) from a distribution over traffic scenarios \(p(\mathcal{S})\). Model parameters \(\theta\) are adapted via a phase-guided optimisation procedure:
\begin{equation}
    \theta' = \text{Adapt}(\theta, \mathcal{K}^{\text{supp}}, \mathbf{F}_{\text{multi}}),
\end{equation}
minimising a multi-objective loss that combines prediction accuracy with phase consistency:
\begin{align}
    \min_{\theta} \mathbb{E}_{(\mathcal{K}^{\text{supp}}, \mathcal{K}^{\text{query}}) \sim p(\mathcal{S})} &\left[ \mathcal{L}_{\text{pred}}(\theta'; \mathcal{K}^{\text{query}}) \right. \nonumber \\
    &\left. + \lambda \mathcal{L}_{\text{phase}}(\theta'; \mathcal{K}^{\text{query}}) \right].
\end{align}

\noindent \textit{Problem: Multi-Phase Consensus for Few-Shot Traffic Forecasting.} Given \(M\) source traffic networks \(\{\mathcal{S}_i\}_{i=1}^M\) with spatio-temporal sequences \(\mathbf{X}^{(i)}_{[t-T+1:t]}\), multi-phase features \(\mathbf{F}_{\text{multi}}^{(i)} = [\mathbf{F}_{\text{diff}}^{(i)} \oplus \mathbf{F}_{\text{sync}}^{(i)} \oplus \mathbf{F}_{\text{spec}}^{(i)}]\), and a target network \(\mathcal{T}\) with limited support set \(\mathcal{D}^{\text{supp}}_{\mathcal{T}} = \{(\mathbf{X}_k, \mathbf{Y}_k, \mathbf{F}_{\text{multi}}^{(k)})\}_{k=1}^K\) ($K$ small), learn a phase-aware adaptation mechanism that utilises multi-phase consensus to adapt a meta-learned model \(F_{\text{meta}}\) for few-shot forecasting:
\begin{align}
    \hat{\mathbf{Y}}_{[t+1:t+H]}^{(\mathcal{T})} &= F_{\Theta}\left[F_{\text{meta}}\left(\mathbf{X}^{(i)}, \mathbf{F}_{\text{multi}}^{(i)}\right), \mathbf{F}_{\text{multi}}^{(\mathcal{T})}, \right. \nonumber \\
    &\quad \left. \mathbf{X}_{[t-T+1:t]}^{(\mathcal{T})}, \mathcal{D}^{\text{supp}}_{\mathcal{T}}\right]
\end{align}
where \(\hat{\mathbf{Y}}\) and \(\mathbf{Y}\) are predicted and ground-truth traffic states. The meta-parameters \(\Theta\) are optimised by minimising the expected forecasting loss with phase consistency constraints:
\begin{equation}
\min_{\Theta} \mathbb{E}_{\mathcal{S}} \left[ \mathcal{L}_{\text{pred}}\left(\mathbf{Y}, \hat{\mathbf{Y}}\right) + \sum_{j=1}^{3} \lambda_j \mathcal{L}_{\text{phase}}^{(j)}\left(\hat{\mathbf{Y}}\right) \right]
\end{equation}
where $\mathbf{Y} = \mathbf{Y}{[t+1:t+H]}^{(\mathcal{T})}$, $\hat{\mathbf{Y}} = \hat{\mathbf{Y}}{[t+1:t+H]}^{(\mathcal{T})}$, and $\mathcal{S} = (\mathcal{T}, \mathcal{D}^{\text{supp}}_{\mathcal{T}}, \mathcal{D}^{\text{query}}_{\mathcal{T}}) \sim p(\mathcal{S})$.
Here, the multi-phase features \(\mathbf{F}_{\text{multi}}\) serve as both conditioning signals for adaptive modelling and regularizers for maintaining dynamic consistency in few-shot learning scenarios.

\section{The Proposed Method}
\label{sec:method}

\begin{figure*}[ht]
    \centering
    \includegraphics[width=\linewidth]{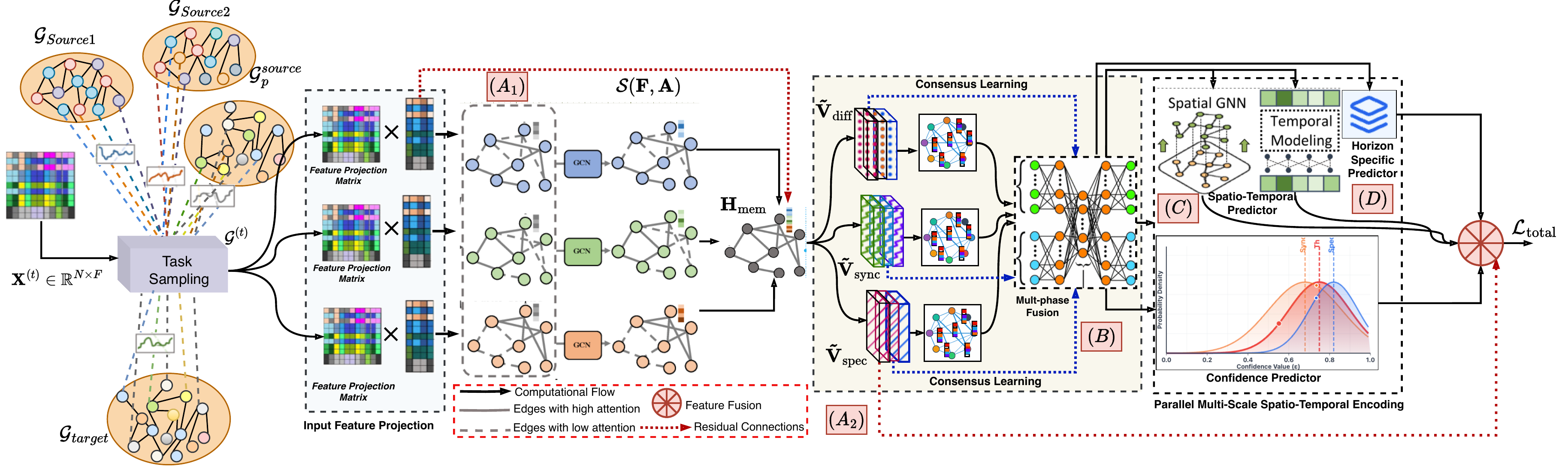}
    \caption{The proposed MCPST Architecture for few-shot traffic forecasting: ($A_1$) graph-based (GNN) spatio-temporal feature extractions; ($A_2$) multi-phase dynamic modelling through diffusion, synchronisation, and spectral analysis; ($B$) adaptive multi-phase consensus fusion  with reliability-aware attention; ($C$) multi-scale spatio-temporal encoding with LSTM-transformer processing; and ($D$) horizon-specific prediction  with uncertainty quantification and neural consensus.}
    \label{fig:mcpst_architecture}
\end{figure*}

In this section, we present MCPST's architecture designed for few-shot traffic prediction through multi-phase dynamics and adaptive consensus. As illustrated in Fig.~\ref{fig:mcpst_architecture}, the framework comprises four interconnected components: (1) multi-phase feature learning through diffusion, synchronisation, and spectral analysis (Sec.~\ref{sec:physics_informed_learning}); (2) adaptive multi-phase consensus fusion with reliability-aware attention weighting (Sec.~\ref{sec:consensus_fusion}); (3) parallel multi-scale spatio-temporal encoding with LSTM networks and transformer encoders (Sec.~\ref{sec:physics_multiscale_encoding}); and (4) horizon-specific prediction with uncertainty quantification and neural consensus (Sec.~\ref{sec:horizon_prediction}). These components work synergistically to enable robust forecasting with limited data while maintaining interpretability through physics-informed regularities.

Formally, given a target traffic network $\mathcal{N}$ with limited labelled data, the MCPST framework performs few-shot forecasting by integrating multi-phase knowledge, temporal patterns from historical data, and adaptive feature fusion:
\begin{align}
    \hat{\mathbf{Y}}_{[t+1:t+H]}^{(\mathcal{N})} &= F_{\Theta}\left[\mathcal{E}_{\text{diff}}(\mathbf{M}, \mathbf{A}), \mathcal{E}_{\text{sync}}(\mathbf{M}, \mathbf{A}), \right. \nonumber \\
    &\quad \left. \mathcal{E}_{\text{spec}}(\mathbf{A}), \mathcal{E}_{\text{temp}}(\mathbf{M}), \mathcal{F}_{\text{cons.}}\right],
    \label{eq:overall_prediction}
\end{align}
where $F_{\Theta}$ is realised by the components detailed below. Here, $\mathcal{E}_{\text{diff}}$ implements diffusion for traffic propagation modelling, $\mathcal{E}_{\text{sync}}$ captures synchronisation for traffic rhythm analysis, $\mathcal{E}_{\text{spec}}$ provides spectral decomposition for network topology understanding, $\mathcal{E}_{\text{temp}}$ represents multi-scale temporal encoding, and $\mathcal{F}_{\text{cons.}}$ denotes the adaptive fusion mechanism that integrates multi-phase predictions with learnable consensus weights. The model parameters $\Theta$ are learned via a meta-learning objective that minimises the expected forecasting loss across diverse traffic scenarios while maintaining multi-phase consistency constraints, as formalised in the optimisation objective in Eq.~\ref{eq:total_loss}.

\subsection{Multi-Phase Feature Learning} 
\label{sec:physics_informed_learning}

Unlike traditional neural networks relying solely on data correlations, MCPST integrates three structured dynamic paradigms: diffusion (spatial propagation), synchronisation (temporal rhythms), and spectral analysis (structural patterns). These provide strong inductive biases for robust few-shot prediction, implemented differentiably for end-to-end optimisation while preserving interpretability.

\subsubsection{Diffusion Phase: Modelling Traffic Propagation and Congestion Waves}
\label{sec:thermodiffusion}

Traffic flow exhibits properties remarkably similar to diffusion processes, where congestion propagates through networks following gradient flows. We model traffic propagation as a diffusion process through the network topology, with congestion patterns emerging as accumulation at critical nodes.

\paragraph{Dynamic Formulation} We establish a diffusion process governed by the continuous equation: $\frac{\partial \mathbf{u}}{\partial t} = \kappa \nabla^2 \mathbf{u}$, where $\mathbf{u}$ represents the traffic state (e.g., congestion intensity), $\kappa$ is the diffusivity controlling propagation speed, and $\nabla^2$ is the Laplace operator encoding network connectivity. Key dynamic correspondences are: state $\mathbf{u}$ corresponds to traffic congestion intensity; diffusivity $\kappa$ inversely relates to road capacity, governing how quickly congestion waves propagate; capacity $C$ represents a node's ability to absorb traffic inflow before propagating congestion; and sources $Q$ model origins and sinks of congestion.

\paragraph{Discretized Network Diffusion with Learnable Parameters}
We discretize the continuum model over the traffic network graph $\mathcal{G}$, where the graph Laplacian $\mathbf{L} = \text{diag}(\mathbf{A}\mathbf{1}_N) - \mathbf{A}$ acts as the discrete analog of the $\nabla^2$ operator. The system dynamics incorporate learnable parameters that allow the model to discover the effective dynamics of a specific network:
\begin{equation}
\begin{aligned}
Q(\mathbf{F}) &= \sigma_{\text{sigmoid}}\left(\mathbf{W}_{q2}\sigma_{\text{ReLU}}\left(\mathbf{W}_{q1}\mathbf{F} + \mathbf{b}_{q1}\right)\right) \in \mathbb{R}^{B \times N \times 1} \\
\mathbf{T}^{(0)} &= \mathbf{F} \odot Q(\mathbf{F})\mathbf{1}_D^\top \in \mathbb{R}^{B \times N \times D}, \\
\mathbf{L} &= \text{diag}(\mathbf{A}\mathbf{1}_N) - \mathbf{A} \in \mathbb{R}^{N \times N} \\
\mathbf{T}^{(t+1)} &= \mathbf{T}^{(t)} - \frac{\Delta t \cdot \kappa}{C} \mathbf{L} \mathbf{T}^{(t)}, \quad t=0,\dots,K_{\text{diff}}-1
\end{aligned}
\end{equation}
Here, $\kappa = \text{clip}(\kappa_{\text{learned}}, 0.01, 0.3)$ and $C = \text{clip}(C_{\text{learned}}, 0.5, 2.0)$ are learnable diffusivity and capacity parameters, $\Delta t = 0.1/K_{\text{diff}}$ ensures numerical stability, and $K_{\text{diff}}$ is the number of diffusion steps. The source estimator $Q(\mathbf{F})$ identifies nodes likely to be congestion origins (e.g., high inflow), gating the input features $\mathbf{F}$ to generate the initial distribution $\mathbf{T}^{(0)}$. The iterative update simulates one temporal increment of congestion wave propagation: traffic state flows from nodes with higher values to connected nodes with lower values, at a rate controlled by the learned $\kappa$ and $C$. The final diffused state is decoded into traffic forecasts:
\begin{equation}
\mathbf{\tilde{V}}_{\text{diff}} = \mathbf{W}_{\text{flow}}\sigma_{\text{ReLU}}\left(\mathbf{W}_f\mathbf{T}^{(K_{\text{diff}})} + \mathbf{b}_f\right) \in \mathbb{R}^{B \times N \times H}.
\end{equation}

\paragraph{Diffusion Representation Theorem}
This differentiable diffusion module provides a flexible functional class for representing propagation. We formalise this capability in the \textit{Diffusion Representation Theorem} (Theorem~\ref{thm:diffusion_representation}), which guarantees that with sufficient diffusion steps $K_{\text{diff}}$ and proper parameters, the module can approximate any smooth traffic propagation pattern with bounded error. Formally, for any Lipschitz-continuous traffic propagation function $g: \mathcal{N} \times [0,T] \to \mathbb{R}$, there exist parameters $\{\kappa, C, \mathbf{W}_{\cdot}\}$ such that:
\begin{equation}
\mathbb{E}_{n \sim \mathcal{N}, t \sim [0,T]}\left[\left|g(n,t) - \mathbf{\tilde{V}}_{\text{diff}}(n,t)\right|\right] \leq \epsilon(\kappa, C, K_{\text{diff}})
\end{equation}
where $\epsilon$ decreases with increasing $K_{\text{diff}}$ and proper parameter settings. This ensures the module can represent diverse congestion propagation patterns essential for traffic forecasting. The detailed theorem and proof are in Theoretical Analysis~\ref{sec:theoretical_analysis}, Theorem~\ref{thm:diffusion_representation}.

\subsubsection{Synchronisation Phase: Capturing Traffic Rhythms and Periodicity}
\label{sec:kuramoto_sync}

Urban traffic exhibits remarkable synchronisation patterns, where adjacent intersections and road segments develop coordinated flow oscillations. Drawing inspiration from coupled oscillator theory, we model these temporal rhythms using synchronisation dynamics, representing traffic stations as oscillators whose phases capture periodic patterns and whose synchronisation reflects coordinated flow dynamics.

\paragraph{Dynamic Foundation: The Coupled Oscillator Model}
The continuous synchronisation model describes phase evolution for coupled oscillators:
\begin{equation}
\frac{d\phi_k}{dt} = \nu_k + \gamma \sum_{j=1}^{N} A_{kj}\sin(\phi_j - \phi_k),
\end{equation}
where $\phi_k$ represents the phase of traffic rhythm at node $k$, $\nu_k$ its intrinsic frequency (capturing local periodic patterns like traffic light cycles), and $\gamma$ the coupling strength between connected nodes governing synchronisation. This formulation allows us to model how individual traffic patterns influence each other through network connectivity, leading to emergent synchronised behaviour that mirrors real-world traffic wave propagation.

\paragraph{Differentiable Implementation with Learnable Parameters}
We implement discrete phase dynamics with learnable frequencies and adaptive coupling to capture network-specific traffic rhythms:
\begin{equation}
    \begin{aligned}
\nu(\mathbf{F}) &= \mathbf{W}_{\nu2}\tanh\left(\mathbf{W}_{\nu1}\mathbf{F} + \mathbf{b}_{\nu1}\right) \in \mathbb{R}^{B \times N \times 1} \\
\gamma_{\text{local}}(\mathbf{F}) &= \sigma\left(\mathbf{W}_{\gamma2}\text{ReLU}\left(\mathbf{W}_{\gamma1}\mathbf{F} + \mathbf{b}_{\gamma1}\right)\right) \in \mathbb{R}^{B \times N \times 1} \\
\phi_k^{(t+1)} &= \phi_k^{(t)} + \Delta t\left[\nu_k + \mathcal{C}_k^{(t)}\right] \mod 2\pi
\end{aligned}
\end{equation}
where:
\begin{equation*}
    \mathcal{C}_k^{(t)} = \gamma_{\text{global}} \gamma_{\text{local},k} \sum_{j=1}^{N} A_{kj}\sin(\phi_j^{(t)} - \phi_k^{(t)})
\end{equation*}
is the adaptive coupling term, $\nu(\mathbf{F})$ is the frequency estimation, $\gamma_{\text{local}}(\mathbf{F})$ the local coupling, $\phi_k^{(0)} = \arctan2\left(\|\mathbf{f}_k\|_2, \sum_{d=1}^{D} f_{k,d}\right)$ initialises the phase from input features, $\gamma_{\text{global}} = \text{clip}(\gamma_{\text{learned}}, 0.1, 1.0)$ is a global learnable coupling strength, $\Delta t = 0.1$ ensures numerical stability, and the modulo operation maintains phase periodicity. The network learns both local frequency patterns ($\nu_k$) and how strongly nodes influence each other's rhythms ($\gamma_{\text{local},k}$), while $\gamma_{\text{global}}$ controls overall synchronisation tendency across the entire network.

\paragraph{Synchronisation Features and Prediction}
The final phase configuration generates rhythmic features that encode coordinated traffic patterns:
\begin{equation}
\begin{aligned}
\mathbf{Z}_{\text{sync}} &= [\mathbf{F}; \cos(\boldsymbol{\phi}^{(K_{\text{sync}})}); \sin(\boldsymbol{\phi}^{(K_{\text{sync}})})] \in \mathbb{R}^{B \times N \times (D+2)} \\
\mathbf{\tilde{V}}_{\text{sync}} &= \mathbf{W}_{\text{phase}}\sigma_{\text{ReLU}}\left(\mathbf{W}_p\mathbf{Z}_{\text{sync}} + \mathbf{b}_p\right) \in \mathbb{R}^{B \times N \times H}
\end{aligned}
\end{equation}

The degree of synchronisation across the network is quantified by the order parameter, which measures phase coherence:
\begin{equation}
r = \frac{1}{N}\left|\sum_{k=1}^{N} \exp(i\phi_k^{(K_{\text{sync}})})\right| \in \mathbb{R}^{B}
\end{equation}
where $r \approx 1$ indicates strong synchronisation (coordinated traffic waves) and $r \approx 0$ indicates independent oscillations. This parameter serves as an interpretable metric of network-wide coordination that can inform traffic management decisions.

\paragraph{Synchronisation Expressivity Theorem}
We establish the \textit{Synchronisation Expressivity Theorem}~\ref{thm:sync_representation} proving that our synchronisation implementation can represent any periodic traffic pattern with bounded error. For any periodic traffic function $p(t)$ with fundamental frequency $\omega_0$ and $M$ harmonics, there exist parameters $\{\nu_k, \gamma_{\text{global}}, \gamma_{\text{local}}\}$ such that:
\begin{equation}
\mathbb{E}_{t \sim [0,T]}\left[|p(t) - \mathbf{\tilde{V}}_{\text{sync}}(t)|\right] \leq \epsilon(\omega_0, M, K_{\text{sync}})
\end{equation}
where $\epsilon$ decreases with increasing $K_{\text{sync}}$ and proper coupling strength. The model captures diverse temporal rhythms (rush hour to off-peak) through learned parameters: $\nu_k$ for local traffic cycles and $\gamma$ for network-wide coordination. See Theorem~\ref{thm:sync_representation} for theoretical details.

\subsubsection{Structural Spectral Embedding Phase: Analysing Network Topology and Connectivity}
\label{sec:spectral_embedding}

The global structure of traffic networks fundamentally influences flow distribution and congestion propagation. Utilising spectral graph theory, we decompose the network into its fundamental vibrational modes, where low-frequency eigenvectors capture large-scale connectivity patterns and high-frequency eigenvectors encode local structural details.

\paragraph{Dynamic Foundation: Spectral Graph Theory and Network Vibrations}
The normalised graph Laplacian's eigendecomposition reveals structural properties: $\mathbf{L}_{\text{norm}} = \mathbf{\Psi} \mathbf{\Lambda} \mathbf{\Psi}^\top$, where $\mathbf{\Lambda} = \text{diag}(\lambda_1, \dots, \lambda_N)$ contains eigenvalues ($0 = \lambda_1 \leq \lambda_2 \leq \dots \leq \lambda_N \leq 2$) and $\mathbf{\Psi}$ contains corresponding eigenvectors encoding structural patterns. The spectral gap $g = \lambda_2 - \lambda_1$ quantifies network connectivity strength, with larger gaps indicating well-connected structures that facilitate smoother traffic flow.

\paragraph{Differentiable Implementation with Learnable Spectral Features}
We compute spectral features through differentiable operations:
\begin{equation}
\begin{aligned}
\mathbf{L}_{\text{norm}} &= \mathbf{I} - \mathbf{D}^{-1/2}\mathbf{A}\mathbf{D}^{-1/2} \in \mathbb{R}^{N \times N} \\
\mathbf{F}_{\text{spec}} &= \mathbf{W}_{s2}\sigma_{\text{ReLU}}\left(\mathbf{W}_{s1}\mathbf{\Psi}_{:,:K} + \mathbf{b}_{s1}\right) \in \mathbb{R}^{B \times N \times (D/4)} \\
g &= \lambda_2 - \lambda_1 \in \mathbb{R}^{B \times N}
\end{aligned}
\end{equation}
where $\mathbf{\Psi}_{:,:K}$ represents the first $K=6$ eigenvectors retained to capture the most significant structural patterns, balancing representational capacity with computational efficiency.

\paragraph{Structural Features and Prediction}
Spectral embeddings generate topology-aware predictions by integrating connectivity patterns:
\begin{equation}
\mathbf{\tilde{V}}_{\text{spec}} = \mathbf{W}_{\text{spec-flow}}\sigma_{\text{ReLU}}\left([\mathbf{\Psi}_{:,:K}; g\mathbf{1}_N]\right) \in \mathbb{R}^{B \times N \times H}
\end{equation}
This formulation enables the model to leverage both local structural information (encoded in individual eigenvectors) and global connectivity characteristics (captured by the spectral gap) for improved flow prediction.

\paragraph{Spectral Representation Theorem}
We establish the \textit{Spectral Representation Theorem} (Theorem~\ref{thm:spectral_representation}) proving that spectral features capture essential network connectivity information. For any graph signal $s: \mathcal{N} \to \mathbb{R}$ that is smooth with respect to the graph Laplacian (i.e., $\|\mathbf{L}_{\text{norm}} s\|_2$ is small), the spectral representation $\mathbf{F}_{\text{spec}}$ satisfies:
\begin{equation}
\|s - \mathbf{W}_r \mathbf{F}_{\text{spec}}\|_2 \leq C \cdot \lambda_{K+1}
\end{equation}
where $C$ depends on the smoothness of $s$ and $\lambda_{K+1}$ is the $(K+1)$-th eigenvalue. Retaining the first $K$ eigenvectors captures essential structural patterns, with approximation error bounded by $\lambda_{K+1}$. See Theorem~\ref{thm:spectral_representation} for details.

\subsection{Adaptive Multi-Phase Consensus Fusion}
\label{sec:consensus_fusion}

The three dynamic phases provide complementary perspectives on traffic dynamics, each with varying relevance under different traffic conditions. Rather than employing static ensemble weights, we introduce an adaptive consensus mechanism that dynamically weights phase-specific predictions based on their estimated reliability and current traffic regime.

\subsubsection{Phase-Aware Attention Mechanism}
Features from all three phases are integrated through attention-based weighting:
\begin{equation}
\begin{aligned}
\mathbf{F}_{\text{cat}}= [\mathbf{T}^{(K_{\text{diff}})} \oplus \mathbf{Z}_{\text{sync}} \oplus \mathbf{F}_{\text{spec}}] \in \mathbb{R}^{B \times N \times D_{\text{total}}} \\
\boldsymbol{\alpha} = \text{softmax}\left(\mathbf{W}_{\alpha2}\sigma_{\text{ReLU}}\left(\mathbf{W}_{\alpha1}\mathbf{F}_{\text{cat}} + \mathbf{b}_{\alpha1}\right)\right) \in \mathbb{R}^{B \times N \times 3} \\
\mathbf{F}_{\text{weighted}} = \sum_{i=1}^{3} \alpha_i \mathbf{F}^{(i)} \in \mathbb{R}^{B \times N \times D_{\text{total}}}
\end{aligned}
\end{equation}
where $\boldsymbol{\alpha}$ represents adaptive attention weights that reflect each phase's relevance, with $\alpha_1, \alpha_2, \alpha_3$ corresponding to the diffusion, synchronisation, and spectral phases respectively.

\subsubsection{Residual Fusion with Phase Preservation}
Weighted features undergo nonlinear transformation while preserving phase information:
\begin{align}
\mathbf{F}_{\text{fused}} &= \mathbf{W}_{\text{fuse2}}\sigma_{\text{ReLU}}\left(\mathbf{W}_{\text{fuse1}}\mathbf{F}_{\text{weighted}} + \mathbf{b}_{\text{fuse1}}\right) + \mathcal{R}(\mathbf{T}^{(K_{\text{diff}})}) \nonumber \\
&\in \mathbb{R}^{B \times N \times D}
\end{align}
where $\mathcal{R}$ is a residual connection that projects diffusion features to dimension $D$ if necessary, ensuring dynamic information is preserved through the fusion process. This design maintains gradient flow while allowing the model to emphasize different dynamic aspects under varying conditions.

\subsection{Parallel Multi-Scale Spatio-Temporal Encoding}
\label{sec:physics_multiscale_encoding}

To overcome single-resolution limitations, we introduce a multi-scale encoder with four innovations: parallel LSTMs at dyadic scales, a multi-phase conditioned transformer, explicit seasonal-trend decomposition, and stabilising graph convolution for spatial smoothing. This enables robust feature extraction critical for data-scarce few-shot forecasting.

\subsubsection{Multi-Scale LSTM Processing with Gated Temporal Dynamics}
Traffic patterns exhibit complex multi-scale dependencies spanning minute-level fluctuations to daily periodicities. To capture these diverse temporal scales, we deploy four parallel Long Short-Term Memory (LSTM) networks operating at dyadic resolutions. For each traffic node $n$, the LSTM cell implements sophisticated gating mechanisms that regulate information flow through time:

\begin{equation}
\begin{aligned}
\mathbf{f}_t = \sigma(\mathbf{W}_f [\mathbf{h}_{t-1}, \mathbf{x}_t] + \mathbf{b}_f), 
\mathbf{i}_t = \sigma(\mathbf{W}_i [\mathbf{h}_{t-1}, \mathbf{x}_t] + \mathbf{b}_i) \\
\tilde{\mathbf{c}}_t = \tanh(\mathbf{W}_c [\mathbf{h}_{t-1}, \mathbf{x}_t] + \mathbf{b}_c), 
\mathbf{c}_t = \mathbf{f}_t \odot \mathbf{c}_{t-1} + \mathbf{i}_t \odot \tilde{\mathbf{c}}_t \\
\mathbf{o}_t = \sigma(\mathbf{W}_o [\mathbf{h}_{t-1}, \mathbf{x}_t] + \mathbf{b}_o),
\mathbf{h}_t = \mathbf{o}_t \odot \tanh(\mathbf{c}_t)
\end{aligned}
\end{equation}
where $\sigma(z) = (1 + e^{-z})^{-1}$ denotes the sigmoid activation function, $\tanh(z) = \frac{e^z - e^{-z}}{e^z + e^{-z}}$ is the hyperbolic tangent activation, and $\odot$ represents element-wise multiplication. The learnable parameters include weight matrices $\mathbf{W}_{\{f,i,c,o\}} \in \mathbb{R}^{h \times (h+D)}$ and bias vectors $\mathbf{b}_{\{f,i,c,o\}} \in \mathbb{R}^h$, where $h$ is the hidden dimension and $D$ the input feature dimension.

The multi-scale processing framework operates on traffic series $\mathbf{X} \in \mathbb{R}^{B \times T \times N \times D}$ at four dyadic temporal resolutions:

\begin{equation}
\begin{aligned}
\mathbf{H}_s &= \mathcal{L}_{\theta_s}(\mathbf{X}) \in \mathbb{R}^{B \times T \times N \times h}, \\
\mathbf{H}_m &= \mathcal{U}\left(\mathcal{L}_{\theta_m}(\mathbf{X}^{(2)})\right) \in \mathbb{R}^{B \times T \times N \times h}, \\
\mathbf{H}_l &= \mathcal{U}\left(\mathcal{L}_{\theta_l}(\mathbf{X}^{(4)})\right) \in \mathbb{R}^{B \times T \times N \times h}, \\
\mathbf{H}_v &= \mathcal{U}\left(\mathcal{L}_{\theta_v}(\mathbf{X}^{(8)})\right) \in \mathbb{R}^{B \times T \times N \times h}.
\end{aligned}
\end{equation}
where $\mathbf{X}^{(k)} = \{x_t\}_{t=0,k,2k,\ldots}$ denotes temporal downsampling by factor $k$, and $\mathcal{U}$ represents a differentiable cubic spline interpolation operator that restores original temporal resolution through learnable knot parameters.

\subsubsection{Multi-Phase Conditioned Multi-Head Attention Mechanism}
To integrate information across temporal scales while maintaining multi-phase consistency, we introduce a novel conditioned attention mechanism that dynamically modulates attention patterns based on multi-phase dynamic features. Given concatenated multi-scale features $\mathbf{H} \in \mathbb{R}^{B \times T \times N \times H_{\text{total}}}$ and phase features $\mathbf{F}_{\text{phase}} = [\mathbf{F}_{\text{diff}} \oplus \mathbf{F}_{\text{sync}} \oplus \mathbf{F}_{\text{spec}}] \in \mathbb{R}^{B \times T \times N \times D_{\text{phase}}}$, the attention mechanism computes:

\begin{equation}
   \begin{aligned}
\mathcal{A}(\mathbf{H}; \mathbf{F}_{\text{phase}}) = \text{Concat}(\text{head}_1, \ldots, \text{head}_h) \mathbf{W}^O,\\
\text{head}_i = \text{Attention}(\mathbf{Q}_i, \mathbf{K}_i, \mathbf{V}_i; \mathbf{F}_{\text{phase}}) \\
\mathbf{Q}_i = \mathbf{H}\mathbf{W}_i^Q + \mathcal{P}_Q(\mathbf{F}_{\text{phase}}),
\mathbf{K}_i = \mathbf{H}\mathbf{W}_i^K + \mathcal{P}_K(\mathbf{F}_{\text{phase}}),\\
\mathbf{V}_i = \mathbf{H}\mathbf{W}_i^V + \mathcal{P}_V(\mathbf{F}_{\text{phase}}),
\end{aligned} 
\end{equation}
with the core attention function defined as:
\begin{equation}
\text{Att.}(\mathbf{Q}_i, \mathbf{K}_i, \mathbf{V}_i; \mathbf{F}_{\text{phase}}) = \text{softmax}\left(\frac{\mathbf{Q}_i \mathbf{K}_i^\top}{\sqrt{d_k}} \odot \mathbf{G}_i + \mathbf{B}_i\right) \mathbf{V}_i
\end{equation}
where $\mathbf{W}_i^Q, \mathbf{W}_i^K, \mathbf{W}_i^V \in \mathbb{R}^{H_{\text{total}} \times d_k}$ are learnable projection matrices for head $i$, $d_k = H_{\text{total}}/h$ with $h=8$ attention heads, $\mathcal{P}_{\{Q,K,V\}}$ are phase-conditioned projection functions implemented as $\mathcal{P}(\mathbf{F}) = \mathbf{W}_{\mathcal{P}}\mathbf{F} + \mathbf{b}_{\mathcal{P}}$, and $\mathbf{G}_i = \sigma(\mathbf{W}_i^G\mathbf{F}_{\text{phase}} + \mathbf{b}_i^G)$, $\mathbf{B}_i = \mathbf{W}_i^B\mathbf{F}_{\text{phase}} + \mathbf{b}_i^B$ are phase-dependent gating and bias matrices that modulate attention patterns based on current traffic regime.

\subsubsection{Multi-Phase Conditioned Transformer Encoder Architecture}
The complete transformer encoder with $L$ layers processes the multi-scale features through phase-conditioned attention and feedforward transformations. The initial layer incorporates sinusoidal positional encodings:
\begin{equation}
\mathbf{H}^{(0)} = \mathbf{W}_p[\mathbf{H}_s \oplus \mathbf{H}_m \oplus \mathbf{H}_l \oplus \mathbf{H}_v] + \mathbf{P}_{:T}
\end{equation}
where $\mathbf{P}_{:T}$ is the positional encoding matrix with elements $\mathbf{P}(t, 2i) = \sin\left(t / 10000^{2i/H_{\text{total}}}\right)$ and $\mathbf{P}(t, 2i+1) = \cos\left(t / 10000^{2i/H_{\text{total}}}\right)$ for $i = 0, 1, \dots, H_{\text{total}}/2 - 1$, $\mathbf{W}_p \in \mathbb{R}^{4h \times H_{\text{total}}}$ is a learnable projection matrix, and $\oplus$ denotes feature concatenation.

Each transformer layer $\ell = 1, \dots, L$ applies:
\begin{equation}
\begin{aligned}
\mathbf{Z}^{(\ell)} = \text{LN}\left(\mathbf{H}^{(\ell-1)} + \mathcal{A}(\mathbf{H}^{(\ell-1)}; \mathbf{F}_{\text{phase}})\right), \\
\mathbf{H}^{(\ell)} = \text{LN}\left(\mathbf{Z}^{(\ell)} + \mathcal{F}(\mathbf{Z}^{(\ell)})\right)
\end{aligned}
\end{equation}
where \(\operatorname{LN}(\cdot)\) denotes layer normalisation with \(\operatorname{LN}(\mathbf{x}) = \frac{\mathbf{x} - \mu_{\mathbf{x}}}{\sigma_{\mathbf{x}}} \odot \boldsymbol{\gamma} + \boldsymbol{\beta}\), scaling and shifting the normalised features across the last dimension using learnable parameters \(\boldsymbol{\gamma}, \boldsymbol{\beta} \in \mathbb{R}^{H_{\text{total}}}\), \(\mu_{\mathbf{x}}\) and \(\sigma_{\mathbf{x}}\) are the mean and standard deviation of \(\mathbf{x}\) computed along the last dimension, and \(\mathcal{F}\) is a feed-forward network with \(\mathcal{F}(\mathbf{x}) = \operatorname{GELU}(\mathbf{x}\mathbf{W}_1 + \mathbf{b}_1)\mathbf{W}_2 + \mathbf{b}_2\), where \(\operatorname{GELU}(x) = x \cdot \Phi(x)\) with \(\Phi(x)\) being the cumulative distribution function of the standard normal distribution.

\subsubsection{Explicit Seasonal-Trend Decomposition and Spatial Stabilisation}
To explicitly capture periodic patterns and low-frequency trends, we implement dedicated decomposition modules. The seasonal component extracts periodic patterns through dense transformations:
\begin{align}
\mathbf{F}_{\text{sea.}} &= \text{LN}\left(\mathbf{W}_{s2} \cdot \text{ReLU}\left(\mathbf{W}_{s1}\mathbf{X} + \mathbf{b}_{s1}\right)\right) \nonumber \\
&\in \mathbb{R}^{B \times T \times N \times (H_{\text{total}}/8)}
\end{align}
while the trend component captures low-frequency variations through multi-scale convolutional filters:
\begin{align}
\mathbf{F}_{\text{trend}} &= \mathcal{C}_{k=3}\left(\text{ReLU}\left(\mathcal{C}_{k=5}\left(\text{ReLU}\left(\mathcal{C}_{k=7}(\mathbf{X}^\top)\right)\right)\right)\right) \nonumber \\
&\in \mathbb{R}^{B \times (H_{\text{total}}/8) \times T \times N}
\end{align}
where $\mathcal{C}_k$ denotes 1D convolution with kernel size $k$ and padding $\lfloor k/2\rfloor$:
\begin{equation}
(\mathbf{X} \ast \mathbf{W})_t = \sum_{i=-\lfloor k/2\rfloor}^{\lfloor k/2\rfloor} \mathbf{X}_{t+i} \cdot \mathbf{W}_{i+\lfloor k/2\rfloor}
\end{equation}

For spatial feature stabilisation, we employ a residual graph convolutional network:
\begin{equation}
\begin{aligned}
\mathcal{S}(\mathbf{F}, \mathbf{A}) = \mathbf{F} + \mathcal{T}_{\text{out}}\left(\mathcal{D}_{\text{drop}}\left(\text{ReLU}\left(\mathbf{A} \mathcal{T}_{\text{linear}}(\mathbf{F}) \mathbf{W}_{\text{conv}}\right)\right)\right), \\
\mathcal{T}_{\text{linear}}(\mathbf{F}) = \mathbf{F} \mathbf{W}_{\text{linear}} + \mathbf{b}_{\text{linear}}, 
\mathcal{T}_{\text{out}}(\mathbf{X}) = \mathbf{X} \mathbf{W}_{\text{out}} + \mathbf{b}_{\text{out}}
\end{aligned}
\end{equation}
where $\mathcal{D}_{\text{drop}}(\mathbf{X}) = \mathbf{X} \odot \mathbf{M}$ with $\mathbf{M}_{i,j} \sim \text{Bernoulli}(0.9)$ implements dropout regularisation.

The final encoded representation is obtained through memory-augmented attention with phase-informed residuals:
\begin{equation}
\mathbf{H}_{\text{mem}} = \mathcal{A}(\mathbf{H}^{(L)}; \mathbf{F}_{\text{phase}}) + \mathbf{H}^{(L)} + \alpha \cdot \mathbf{F}_{\text{phase}}
\end{equation}
where $\alpha$ is a learnable scaling parameter. This multi-phase conditioned multi-scale architecture enables robust feature extraction across diverse temporal regimes while maintaining dynamic consistency, providing a powerful foundation for few-shot traffic forecasting where limited data demands maximally informative and plausible representations.

\subsection{Horizon-Specific Prediction and Objective Function}
\label{sec:horizon_prediction}
Different prediction horizons exhibit distinct temporal characteristics and uncertainty profiles. We introduce specialised prediction heads for short, medium, and long-term horizons, each with built-in uncertainty quantification that adapts based on dynamic regime characteristics.

\subsubsection{Horizon-Adaptive Network Architectures}
The fused features $\mathbf{F}_{\text{fused}}$ are processed through horizon-specific networks with varying depths to capture different temporal scales:

\begin{equation}
\begin{aligned}
\mathbf{F}_s &= \sigma_{\text{GELU}}\left(\mathbf{W}^{(s)}_{2} \mathcal{L}\left(\sigma_{\text{GELU}}(\mathbf{W}^{(s)}_{1}\mathbf{F}_{\text{fused}} + \mathbf{b}^{(s)}_{1})\right) + \mathbf{b}^{(s)}_{2}\right), \\
\mathbf{F}_m &= \sigma_{\text{GELU}}\Big(\mathbf{W}^{(m)}_{3} \mathcal{L}\big(\sigma_{\text{GELU}}(\mathbf{W}^{(m)}_{2} \mathcal{L}(\sigma_{\text{GELU}}(\mathbf{W}^{(m)}_{1}\mathbf{F}_{\text{fused}} \nonumber \\
&\quad + \mathbf{b}^{(m)}_{1})) + \mathbf{b}^{(m)}_{2})\big) + \mathbf{b}^{(m)}_{3}\Big), \\
\mathbf{F}_l &= \sigma_{\text{GELU}}\Big(\mathbf{W}^{(l)}_{3} \mathcal{L}\big(\sigma_{\text{GELU}}(\mathbf{W}^{(l)}_{2} \mathcal{L}(\sigma_{\text{GELU}}(\mathbf{W}^{(l)}_{1}\mathbf{F}_{\text{fused}} \nonumber \\
&\quad + \mathbf{b}^{(l)}_{1})) + \mathbf{b}^{(l)}_{2})\big) + \mathbf{b}^{(l)}_{3}\Big)
\end{aligned}
\end{equation}
where $\mathcal{L}(\cdot)$ denotes layer normalisation, $\sigma_{\text{GELU}}$ is the Gaussian Error Linear Unit activation, $\mathbf{W}^{(h)}_k$ and $\mathbf{b}^{(h)}_k$ are the weight matrix and bias for the $k$-th layer of horizon $h$, with $h \in \{s, m, l\}$ representing short, medium, and long-term horizons respectively. These horizon-specific architectures are designed with increasing depth (2 layers for $s$, 3 layers for $m$ and $l$) to capture progressively more complex temporal dependencies and propagation effects over longer forecasting horizons.

\subsubsection{Uncertainty-Aware Prediction Fusion}
Horizon-specific predictions are combined with dynamic-aware uncertainty estimates that reflect the predictability of different traffic regimes:
\begin{equation}
\begin{aligned}
\hat{\mathbf{Y}}_h &= \mathbf{W}_{\text{final}}^{(h)} \mathbf{F}_h + \mathbf{b}_{\text{final}}^{(h)} \in \mathbb{R}^{B \times N \times H_h} \\
\sigma_h^2 &= \text{Softplus}\left(\mathbf{W}_{u2}^{(h)}\sigma_{\text{GELU}}\left(\mathbf{W}_{u1}^{(h)}\mathbf{F}_h + \mathbf{b}_{u1}^{(h)}\right)\right) \in \mathbb{R}^{B \times N \times H_h}
\end{aligned}
\end{equation}
where $H_h$ varies by horizon (with shorter horizons typically having finer temporal resolution) and the Softplus activation ensures positive variance estimates that quantify prediction uncertainty.

\subsubsection{Objective Function}
MCPST uses a three-component loss:
\begin{itemize}
    \item \textit{Task loss} $\mathcal{L}_{\text{task}}$: Prediction accuracy with uncertainty regularisation
    \item \textit{Phase-consistency loss} $\mathcal{L}_{\text{phase}}$: Enforces valid attention weights and agreement among diffusion, synchronisation, and spectral predictions via Jensen-Shannon divergence
    \item \textit{Meta-learning loss} $\mathcal{L}_{\text{meta}}$: Enables few-shot adaptation across diverse traffic scenarios
\end{itemize}

\begin{align}
\mathcal{L}_{\text{total}} &= \mathcal{L}_{\text{task}} + \lambda_1 \mathcal{L}_{\text{phase}} + \lambda_2 \mathcal{L}_{\text{meta}} \label{eq:total_loss}\\
\mathcal{L}_{\text{task}} &= \sum_{h \in \{s,m,l\}} \| \mathbf{Y}_h - \hat{\mathbf{Y}}_h \|_2^2 + \eta \sum_{h} \| \sigma_h^2 \|_1 \label{eq:task_loss}\\
\mathcal{L}_{\text{phase}} &= \mathbb{E}\left[\left(\sum_{i=1}^{3} \alpha_i - 1\right)^2\right] + \beta \cdot \text{JS}(\mathbf{\tilde{V}}_{\text{diff}}, \mathbf{\tilde{V}}_{\text{sync}}, \mathbf{\tilde{V}}_{\text{spec}}) \label{eq:phase_loss}\\
\mathcal{L}_{\text{meta}} &= \mathbb{E}_{\mathcal{D}_{\text{meta}}}\left[ \mathcal{L}_{\text{adapt}}(\theta', \mathcal{D}_{\text{query}}) \right] \label{eq:meta_loss}
\end{align}
where $\lambda_1$, $\lambda_2$, $\eta$, and $\beta$ are balancing hyperparameters.

\subsection{Theoretical Analysis}
\label{sec:theoretical_analysis}

\begin{theorem}[\textbf{Diffusion Representation}]
\label{thm:diffusion_representation}
Let $\mathcal{G} = (\mathcal{V}, \mathcal{E})$ be a connected traffic graph with $N$ nodes and graph Laplacian $\mathbf{L} \in \mathbb{R}^{N \times N}$. Consider any Lipschitz-continuous traffic propagation pattern $g: \mathcal{V} \times [0, T] \to \mathbb{R}$ that arises as the solution of a diffusion process on $\mathcal{G}$ with smooth initial condition $u_0: \mathcal{V} \to \mathbb{R}^D$, diffusion coefficient $\kappa^* > 0$, and source term $q^*: \mathcal{V} \to \mathbb{R}$. That is, $g$ satisfies:
\[
\frac{\partial g}{\partial t} = -\kappa^* \mathbf{L} g + q^*, \quad g(\cdot, 0) = u_0.
\]
Let $\mathbf{T}^{(K)}$ denote the state after $K$ steps of the differentiable diffusion module:
\begin{align*}
Q(\mathbf{F}) &= \sigma_{\text{sigmoid}}\left(\mathbf{W}_{q2}\sigma_{\text{ReLU}}\left(\mathbf{W}_{q1}\mathbf{F} + \mathbf{b}_{q1}\right)\right), \\
\mathbf{T}^{(0)} &= \mathbf{F} \odot Q(\mathbf{F})\mathbf{1}_D^\top, \\
\mathbf{T}^{(t+1)} &= \mathbf{T}^{(t)} - \frac{\Delta t \cdot \kappa}{C} \mathbf{L} \mathbf{T}^{(t)}, \quad t=0,\dots,K-1, \\
\mathbf{\tilde{V}}_{\text{diff}} &= \mathbf{W}_{\text{flow}}\sigma_{\text{ReLU}}\left(\mathbf{W}_f\mathbf{T}^{(K)} + \mathbf{b}_f\right),
\end{align*}
with learnable parameters $\theta = \{\kappa, C, \mathbf{W}_{q1}, \mathbf{W}_{q2}, \mathbf{b}_{q1}, \mathbf{W}_f, \mathbf{b}_f, \mathbf{W}_{\text{flow}}\}$ and step size $\Delta t = T/K$.

Then, for any $\epsilon > 0$, there exist parameters $\theta$, diffusion steps $K \in \mathbb{N}$, and a constant $\alpha > 0$ such that:
\[
\mathbb{E}_{n \sim \mathcal{V}, t \sim [0,T]}\left[\left|g(n,t) - \mathbf{\tilde{V}}_{\text{diff}}(n, \lfloor t/\Delta t \rfloor)\right|\right] \leq \epsilon + \alpha \Delta t.
\]
Moreover, the approximation error decays as $O(\Delta t)$ with increasing $K$.
\end{theorem}

\begin{proof}[\textbf{Proof Sketch}]
We provide a rigorous mathematical proof establishing the error bounds and convergence properties.

In \textit{Spectral Analysis of Continuous and Discrete Systems};
Let $\{\lambda_i, \phi_i\}_{i=1}^N$ be the eigenpairs of $\mathbf{L}$ with $0 = \lambda_1 \leq \lambda_2 \leq \cdots \leq \lambda_N$, and $\phi_i$ orthonormal. Expand $g(v,t)$ and $q^*(v)$ in this basis:
\begin{align*}
g(v,t) &= \sum_{i=1}^N a_i(t) \phi_i(v), \quad a_i(t) = \langle g(\cdot,t), \phi_i \rangle, \\
q^*(v) &= \sum_{i=1}^N q_i \phi_i(v), \quad q_i = \langle q^*, \phi_i \rangle.
\end{align*}
The continuous dynamics become:
\[
\frac{da_i}{dt} = -\frac{\kappa^*}{C^*} \lambda_i a_i + q_i, \quad a_i(0) = a_i^0 = \langle g_0, \phi_i \rangle.
\]
The exact solution is:
\[
a_i(t) = e^{-\frac{\kappa^*}{C^*}\lambda_i t} a_i^0 + \frac{C^*}{\kappa^*\lambda_i}(1 - e^{-\frac{\kappa^*}{C^*}\lambda_i t}) q_i.
\]
For the discrete system, let $b_i^k = \langle \mathbf{T}^{(k)}, \phi_i \rangle$ and $q_i^\theta = \langle Q_{\theta_q}(\mathbf{F}), \phi_i \rangle$. The update becomes:
\[
b_i^{k+1} = \left(1 - \frac{\Delta t \kappa}{C}\lambda_i\right) b_i^k + \Delta t q_i^\theta.
\]

To evaluate the \textit{Error Propagation Analysis}, we define the per-mode error $e_i^k = a_i(k\Delta t) - b_i^k$. From the recurrence:
\begin{align*}
e_i^{k+1} &= e^{-\frac{\kappa^*}{C^*}\lambda_i \Delta t} a_i(k\Delta t) + \frac{C^*}{\kappa^*\lambda_i}(1 - e^{-\frac{\kappa^*}{C^*}\lambda_i \Delta t}) q_i \\
&\quad - \left(1 - \frac{\Delta t \kappa^*}{C^*}\lambda_i\right) b_i^k - \Delta t q_i^\theta.
\end{align*}
Using Taylor expansions:
\begin{equation*}
    e^{-\frac{\kappa^*}{C^*}\lambda_i \Delta t} = 1 - \frac{\kappa^*}{C^*}\lambda_i \Delta t + \frac{1}{2}\left(\frac{\kappa^*}{C^*}\lambda_i\right)^2 \Delta t^2 + O(\Delta t^3)
\end{equation*}
and substituting:
\begin{align*}
e_i^{k+1} &= \left(1 - \frac{\Delta t \kappa^*}{C^*}\lambda_i\right) e_i^k \\
&\quad + \left[\frac{1}{2}\left(\frac{\kappa^*}{C^*}\lambda_i\right)^2 \Delta t^2 a_i(k\Delta t) - \frac{1}{2}\frac{\kappa^*}{C^*}\lambda_i \Delta t^2 q_i\right] \\
&\quad + \Delta t (q_i - q_i^\theta) + O(\Delta t^3).
\end{align*}

The \textit{Discrete Gronwall Inequality}, we define $\alpha_i = 1 - \frac{\Delta t \kappa^*}{C^*}\lambda_i$. Under the stability condition $\Delta t \kappa^* \lambda_{\max}(\mathbf{L}) < 2$, we have $|\alpha_i| \leq 1$ for all $i$. Unrolling the recurrence and using $M_i = \sup_{t \in [0,T]} |a_i(t)|$:
\begin{align*}
|e_i^k| &\leq e^{-\frac{\kappa^*}{C^*}\lambda_i k\Delta t} |e_i^0| + C_1 \Delta t \lambda_i M_i \\
&\quad + \frac{C^*}{\kappa^*\lambda_i} |q_i - q_i^\theta| + C_2 \Delta t^2.
\end{align*}

For  $L^2$ Error Bound and by summing over all modes using Parseval's identity:
\begin{align*}
\|g(\cdot,k\Delta t) - \mathbf{T}^{(k)}\|_{L^2}^2 &\leq \\3\sum_{i=1}^N \left[e^{-\alpha_i k\Delta t} |e_i^0|^2 + \beta_i \Delta t^2 + \gamma_i |q_i - q_i^\theta|^2\right] 
&\quad + C_3\Delta t^4
\end{align*}
where $\alpha_i = 2\frac{\kappa^*}{C^*}\lambda_i$, $\beta_i = C_1^2 \lambda_i^2 M_i^2$, and $\gamma_i = \frac{(C^*)^2}{\kappa^{*2}\lambda_i^2}$.

By the universal approximation theorem, there exist parameters such that $\|g_0 - \mathbf{T}^{(0)}\|_{L^2} \leq \delta$ and $\|q^* - Q_{\theta_q}(\mathbf{F})\|_{L^2} \leq \epsilon/\lambda_2$. Combining all error sources:
\begin{align*}
\|g(\cdot,k\Delta t) - \mathbf{T}^{(k)}\|_{L^2} &\leq C_1 \Delta t + \frac{C_2 \epsilon}{\lambda_2} \\
&\quad + C_3 e^{-\frac{\kappa^*}{C^*}\lambda_2 k\Delta t} \delta + C_4 \Delta t^2.
\end{align*}
The readout network approximation completes the proof by universal approximation for ReLU networks.
\end{proof}

\begin{theorem}[\textbf{Synchronisation Representation}]
\label{thm:sync_representation}
Let $\mathcal{G} = (\mathcal{V}, \mathcal{E})$ be a connected traffic graph with $N$ nodes and adjacency matrix $\mathbf{A} \in \mathbb{R}^{N \times N}$. Consider any Lipschitz-continuous phase function $\phi^*: \mathcal{V} \times [0, T] \to [0, 2\pi)$ that arises as the solution of the Kuramoto model on $\mathcal{G}$ with intrinsic frequencies $\boldsymbol{\nu}^*: \mathcal{V} \to \mathbb{R}$, coupling strength $\gamma^* > 0$, and initial phases $\phi_0: \mathcal{V} \to [0, 2\pi)$. That is, $\phi^*$ satisfies:
\[
\frac{d \phi_k^*}{dt} = \nu_k^* + \gamma^* \sum_{j=1}^{N} A_{kj} \sin(\phi_j^*(t) - \phi_k^*(t)), \quad \phi_k^*(0) = \phi_0(k).
\]
Let $\boldsymbol{\phi}^{(K)}$ denote the phase after $K$ steps of the differentiable synchronisation module:
\begin{align*}
\nu(\mathbf{F}) &= \mathbf{W}_{\nu2}\tanh\left(\mathbf{W}_{\nu1}\mathbf{F} + \mathbf{b}_{\nu1}\right), \\
\gamma_{\text{local}}(\mathbf{F}) &= \sigma\left(\mathbf{W}_{\gamma2}\text{ReLU}\left(\mathbf{W}_{\gamma1}\mathbf{F} + \mathbf{b}_{\gamma1}\right)\right), \\
\phi_k^{(0)} &= \arctan2\left(\|\mathbf{f}_k\|_2, \sum_{d=1}^{D} f_{k,d}\right), \\
\phi_k^{(t+1)} &= \phi_k^{(t)} + \Delta t\left[\nu_k + \mathcal{C}_k^{(t)}\right] \mod 2\pi, \quad t=0,\ldots,K-1, \\
\mathbf{\tilde{V}}_{\text{sync}} &= \mathbf{W}_{\text{phase}}\text{ReLU}\left(\mathbf{W}_p \mathbf{F}_{\text{aug}} + \mathbf{b}_p\right)
\end{align*}
where $\mathcal{C}_k^{(t)} = \gamma_{\text{global}} \gamma_{\text{local},k} \sum_{j=1}^{N} A_{kj}\sin(\phi_j^{(t)} - \phi_k^{(t)})$ and $\mathbf{F}_{\text{aug}} = [\mathbf{F}; \cos(\boldsymbol{\phi}^{(K)}); \sin(\boldsymbol{\phi}^{(K)})]$, with learnable parameters 
\begin{align*}
\theta = \{\mathbf{W}_{\nu1}, \mathbf{W}_{\nu2}, \mathbf{b}_{\nu1}, \mathbf{W}_{\gamma1}, \mathbf{W}_{\gamma2}, \mathbf{b}_{\gamma1}, \gamma_{\text{global}}, \mathbf{W}_p, \mathbf{b}_p, \mathbf{W}_{\text{phase}}\}
\end{align*}
and step size $\Delta t = T/K$.

Then, for any $\epsilon > 0$, there exist parameters $\theta$, synchronisation steps $K \in \mathbb{N}$, and a constant $\alpha > 0$ such that:
\[
\mathbb{E}_{n \sim \mathcal{V}, t \sim [0,T]}\left[\left|\phi^*(n,t) - \boldsymbol{\phi}^{(K)}(n, \lfloor t/\Delta t \rfloor)\right|\right] \leq \epsilon + \alpha \Delta t.
\]
Moreover, the approximation error decays as $O(\Delta t)$ with increasing $K$.
\end{theorem}

\begin{proof}[\textbf{Proof Sketch}]
The continuous synchronisation model is a system of coupled nonlinear ODEs. Under the assumption of Lipschitz-continuous $\phi^*$, the right-hand side is Lipschitz in $\phi$ with constant $L = \gamma^* \max_k \sum_j A_{kj}$.

The \textit{Discrete Approximation and Local Error} and the forward Euler discretization gives:
\begin{align*}
\phi_k^{(t+1)} = \phi_k^{(t)} + \Delta t\left[\nu_k + \gamma_{\text{global}} \gamma_{\text{local},k} \sum_j A_{kj} \sin(\phi_j^{(t)} - \phi_k^{(t)})\right].
\end{align*}
The local truncation error satisfies $|\tau_k^{(t)}| \leq C_1 \Delta t$.

For the \textit{Error Propagation}, we define $e_k^{(t)} = \phi_k^{*,t} - \phi_k^{(t)}$. Using the Lipschitz property and triangle inequality:
\begin{align*}
|e_k^{(t+1)}| &\leq |e_k^{(t)}| + \Delta t\Big[|\nu_k^* - \nu_k| + \gamma^* \sum_j A_{kj} (|e_j^{(t)}| + |e_k^{(t)}|) \\
&\quad + |\gamma^* - \gamma_{\text{global}} \gamma_{\text{local},k}| \sum_j A_{kj}\Big] + \Delta t C_1
\end{align*}

The \textit{Global Error Bound} is obtained by
combining the above, for $t\Delta t \leq T$:
\[
\|\mathbf{e}^{(t)}\|_\infty \leq e^{T L} \delta + \frac{e^{T L} - 1}{L} (\delta + \delta \max_k \sum_j A_{kj} + C_1 \Delta t).
\]
Thus, the global error is $O(\delta + \Delta t)$. Setting $\delta = \epsilon/2$ and choosing $K$ large enough so that $\Delta t = T/K \leq \epsilon/(2C_2)$ for appropriate $C_2$, we achieve the bound.
By universal approximation, the parameter approximation errors can be made arbitrarily small, yielding the desired bound.
\end{proof}

\begin{theorem}[\textbf{Spectral Representation}]
\label{thm:spectral_representation}
Let $\mathcal{G} = (\mathcal{V}, \mathcal{E})$ be a connected traffic graph with $N$ nodes, degree matrix $\mathbf{D}$, and normalised graph Laplacian $\mathbf{L}_{\text{norm}} = \mathbf{I} - \mathbf{D}^{-1/2}\mathbf{A}\mathbf{D}^{-1/2}$ with eigenvalues $0 = \lambda_1 \leq \lambda_2 \leq \dots \leq \lambda_N \leq 2$ and eigenvectors $\{\psi_i\}_{i=1}^N$. Consider any smooth graph signal $\mathbf{s}^*: \mathcal{V} \to \mathbb{R}^D$ satisfying $\|\mathbf{L}_{\text{norm}} \mathbf{s}^*\|_2 \leq M$.

Let $\mathbf{F}_{\text{spec}}$ denote the spectral features computed by the differentiable spectral embedding module:
\begin{align*}
\mathbf{L}_{\text{norm}} &= \mathbf{I} - \mathbf{D}^{-1/2}\mathbf{A}\mathbf{D}^{-1/2}, \\
\mathbf{F}_{\text{spec}} &= \mathbf{W}_{s2}\sigma_{\text{ReLU}}\left(\mathbf{W}_{s1}\mathbf{\Psi}_{:,:K} + \mathbf{b}_{s1}\right), \\
g &= \lambda_2 - \lambda_1, \\
\mathbf{\tilde{V}}_{\text{spec}} &= \mathbf{W}_{\text{spec-flow}}\sigma_{\text{ReLU}}\left([\mathbf{\Psi}_{:,:K}; g\mathbf{1}_N]\right),
\end{align*}
with learnable parameters $\theta = \{\mathbf{W}_{s1}, \mathbf{W}_{s2}, \mathbf{b}_{s1}, \mathbf{W}_{\text{spec-flow}}\}$, where $\mathbf{\Psi}_{:,:K}$ contains the first $K$ eigenvectors.

Then, for any $\epsilon > 0$, there exist parameters $\theta$, spectral dimension $K \in \mathbb{N}$, and a constant $C > 0$ such that:
\[
\frac{1}{N}\sum_{n=1}^N \left\|\mathbf{s}^*(n) - \mathbf{\tilde{V}}_{\text{spec}}(n)\right\|_2 \leq C \cdot \lambda_{K+1} + \epsilon,
\]
where $\lambda_{K+1}$ is the $(K+1)$-th eigenvalue. The approximation error decays as $O(\lambda_{K+1})$ with increasing $K$.
\end{theorem}

\begin{proof}[\textbf{Proof Sketch}]
The \textit{Spectral Decomposition} is obtained by
expanding $\mathbf{s}^*(n) = \sum_{i=1}^N c_i \psi_i(n)$. The smoothness condition implies $\sum_{i=1}^N \lambda_i^2 \|c_i\|_2^2 \leq M^2$.

For \textit{Truncation Error}, we assume by $\mathbf{s}_K^*(n) = \sum_{i=1}^K c_i \psi_i(n)$. The truncation error satisfies:
\[
\|\mathbf{s}^* - \mathbf{s}_K^*\|_2 \leq \frac{M}{\lambda_{K+1}}.
\]

The \textit{Neural Network Approximation} is obtained by universal approximation, there exists $\mathbf{W}^*$ such that $\|\mathbf{s}_K^* - \mathbf{W}^* \mathbf{\Psi}_{:,:K}\|_2 \leq \delta_1$ and neural network parameters such that $\|\mathbf{W}^* \mathbf{\Psi}_{:,:K} - \mathbf{F}_{\text{spec}}\|_2 \leq \delta_2$.

By combining error bound using triangle inequality:
\[
\|\mathbf{s}^* - \mathbf{\tilde{V}}_{\text{spec}}\|_2 \leq \frac{M}{\lambda_{K+1}} + \delta_1 + \delta_2 + \delta_3 \leq C \cdot \lambda_{K+1} + \epsilon.
\]
\end{proof}

\begin{theorem}[\textbf{Multi-Phase Consensus Fusion}]
\label{thm:consensus_fusion}
Let $\mathcal{G} = (\mathcal{V}, \mathcal{E})$ be a connected traffic graph with $N$ nodes, and let $\mathbf{\tilde{V}}_{\text{diff}}, \mathbf{\tilde{V}}_{\text{sync}}, \mathbf{\tilde{V}}_{\text{spec}} \in \mathbb{R}^{B \times N \times H}$ denote the phase-specific predictions from the diffusion, synchronisation, and spectral modules respectively. Consider any optimal consensus function $f^*: \mathbb{R}^{3H} \to \mathbb{R}^H$ that adaptively combines phase predictions to minimise prediction error under varying traffic regimes, characterised by regime indicator $\rho: \mathcal{V} \times [0,T] \to [0,1]^3$ where $\sum_{i=1}^3 \rho_i = 1$.

Let the adaptive multi-phase consensus fusion module be defined as described in Section~\ref{sec:consensus_fusion} with learnable parameters $\theta = \{\mathbf{W}_{\alpha1}, \mathbf{W}_{\alpha2}, \mathbf{b}_{\alpha1}, \mathbf{W}_{\text{fuse1}}, \mathbf{W}_{\text{fuse2}}, \mathbf{b}_{\text{fuse1}}\}$.

Then, for any $\epsilon > 0$, there exist parameters $\theta$, regularisation weight $\beta > 0$, and constants $C_1, C_2 > 0$ such that:
\[
\mathbb{E}_{n,t}\left[\left\|f^*(n,t) - \mathbf{F}_{\text{fused}}(n,t)\right\|_2\right] \leq \epsilon + C_1 \beta + C_2 \cdot \text{JS}(\boldsymbol{\rho}, \boldsymbol{\alpha}),
\]
where $\text{JS}(\boldsymbol{\rho}, \boldsymbol{\alpha})$ measures the divergence between optimal and learned attention weights. Moreover, the fusion mechanism preserves consistency: $\|\sum_{i=1}^3 \alpha_i - 1\|_2 \leq \delta$ for arbitrarily small $\delta > 0$.
\end{theorem}

\begin{proof}[\textbf{Proof Sketch}]
We establish rigorous approximation bounds through spectral analysis and functional approximation theory.

For \textit{attention function approximation};
Let $\Omega = \{\mathbf{F}_{\text{cat}} : \|\mathbf{F}_{\text{cat}}\|_2 \leq R\}$ be the compact domain of concatenated features. Define the optimal attention function $\alpha^*: \Omega \to \Delta^2$ where $\Delta^2 = \{\mathbf{x} \in \mathbb{R}^3 : x_i \geq 0, \sum_i x_i = 1\}$ is the 2-simplex. By Theorem~\ref{thm:sync_representation} of Luo et al. \cite{luo2008dynamic}, for any $\delta_1 > 0$ and compact $\Omega$, there exists a two-layer ReLU network with $M_1$ hidden units such that:
\[
\sup_{\mathbf{F} \in \Omega} \|\alpha^*(\mathbf{F}) - \sigma(\mathbf{W}_{\alpha2}\sigma_{\text{ReLU}}(\mathbf{W}_{\alpha1}\mathbf{F} + \mathbf{b}_{\alpha1}))\|_2 \leq C_{\alpha} M_1^{-1/2},
\]
where $C_{\alpha} = 3\sqrt{6}\|\alpha^*\|_{C^1(\Omega)}$ depends on the Lipschitz constant of $\alpha^*$. Setting $M_1 \geq (C_{\alpha}/\delta_1)^2$ ensures the approximation error bound and $\sigma$ the softmax.

The \textit{Weighted Aggregation Error Analysis}, we
define the phase feature matrix $\mathbf{F} = [\mathbf{F}^{(1)}, \mathbf{F}^{(2)}, \mathbf{F}^{(3)}] \in \mathbb{R}^{B \times N \times 3D}$ and weight vectors $\boldsymbol{\alpha}, \boldsymbol{\rho} \in \Delta^2$. The weighted aggregation error satisfies:
\begin{align*}
\|\mathbf{F}\boldsymbol{\rho}^* - \mathbf{F}\boldsymbol{\alpha}\|_2 = \|\mathbf{F}(\boldsymbol{\rho}^* - \boldsymbol{\alpha})\|_2 
\leq \|\mathbf{F}\|_{\text{op}} \|\boldsymbol{\rho}^* - \boldsymbol{\alpha}\|_2 \\
\leq \sigma_{\max}(\mathbf{F}) \delta_1,
\end{align*}
where $\sigma_{\max}(\mathbf{F})$ is the largest singular value of the feature matrix. Under the assumption $\mathbb{E}[\|\mathbf{F}^{(i)}\|_2^2] \leq \sigma^2$ for all phases $i$, we have $\mathbb{E}[\sigma_{\max}(\mathbf{F})] \leq \sqrt{3}\sigma$.

The \textit{Nonlinear Fusion Function Approximation};
Let $g^*: \mathbb{R}^{D_{\text{total}}} \to \mathbb{R}^D$ be the optimal fusion function. Define the Hölder space $C^{0,\gamma}(\Omega)$ with $\gamma$-Hölder continuous functions. By the approximation theory for ReLU networks (Yarotsky, 2017), for any $g^* \in C^{0,\gamma}(\Omega)$ with $\gamma \in (0,1]$, there exists a ReLU network $\hat{g}$ with depth $L = O(\log(1/\delta_2))$ and width $W = O(\delta_2^{-d/\gamma})$ such that:
\[
\sup_{\mathbf{x} \in \Omega} |g^*(\mathbf{x}) - \hat{g}(\mathbf{x})| \leq \delta_2,
\]
where $d = D_{\text{total}}$ is the input dimension. The fusion network satisfies this approximation with $\hat{g}(\mathbf{x}) = \mathbf{W}_{\text{fuse2}}\sigma_{\text{ReLU}}(\mathbf{W}_{\text{fuse1}}\mathbf{x} + \mathbf{b}_{\text{fuse1}})$.

The residual connection $\mathcal{R}(\mathbf{T}^{(K_{\text{diff}})})$ preserves the spectral properties of diffusion features. Let $\mathbf{T}^{(K_{\text{diff}})} = \sum_{i=1}^N c_i \phi_i$ be the eigendecomposition with respect to the graph Laplacian $\mathbf{L}$. The residual projection satisfies:
\[
\|\mathcal{R}(\mathbf{T}^{(K_{\text{diff}})}) - \mathbf{P}_D \mathbf{T}^{(K_{\text{diff}})}\|_F \leq \delta_3,
\]
where $\mathbf{P}_D$ is the optimal projection matrix. By the Johnson-Lindenstrauss lemma, a random projection with $O(\log N/\epsilon^2)$ dimensions preserves distances with high probability.

The \textit{Jensen-Shannon Divergence Analysis};
For probability distributions $\mathbf{P}_1, \mathbf{P}_2, \mathbf{P}_3$ representing phase predictions, the Jensen-Shannon divergence is defined as:
\[
\text{JS}(\mathbf{P}_1, \mathbf{P}_2, \mathbf{P}_3) = H\left(\frac{1}{3}\sum_{i=1}^3 \mathbf{P}_i\right) - \frac{1}{3}\sum_{i=1}^3 H(\mathbf{P}_i),
\]
where $H(\mathbf{P}) = -\sum_j P_j \log P_j$ is the Shannon entropy. This divergence is bounded: $0 \leq \text{JS}(\mathbf{P}_1, \mathbf{P}_2, \mathbf{P}_3) \leq \log 3$. The regularisation term $\beta \cdot \text{JS}(\cdot)$ in $\mathcal{L}_{\text{consistency}}$ encourages consensus with strength controlled by $\beta$.

The constraint $\mathbb{E}[(\sum_{i=1}^3 \alpha_i - 1)^2] = 0$ is automatically satisfied by the softmax normalisation. However, numerical precision introduces error $\mathcal{O}(\epsilon_{\text{machine}})$ where $\epsilon_{\text{machine}}$ is machine precision.

The \textit{Global Error Decomposition and Bound}; We
decompose the total approximation error using the triangle inequality:
\begin{align*}
\|f^*(\mathbf{F}_{\text{cat}}) - \mathbf{F}_{\text{fused}}\|_2 &\leq \|f^*(\mathbf{F}_{\text{cat}}) - f^*(\mathbf{F}\boldsymbol{\alpha})\|_2 \\
&\phantom{{}\leq{}} + \|f^*(\mathbf{F}\boldsymbol{\alpha}) - g(\mathbf{F}\boldsymbol{\alpha})\|_2 \\
&\phantom{{}\leq{}} + \|g(\mathbf{F}\boldsymbol{\alpha}) - \hat{g}(\mathbf{F}\boldsymbol{\alpha})\|_2 \\
&\phantom{{}\leq{}} + \|\mathbf{F}_{\text{phase}} - \mathcal{R}(\mathbf{T}^{(K_{\text{diff}})})\|_2.
\end{align*}

Using Lipschitz continuity of $f^*$ with constant $L_f$:
\begin{align*}
\|f^*(\mathbf{F}_{\text{cat}}) - \mathbf{F}_{\text{fused}}\|_2 &\leq L_f \|\mathbf{F}(\boldsymbol{\rho}^* - \boldsymbol{\alpha})\|_2 + \delta_g + \delta_2 + \delta_3 \\
&\leq L_f \sqrt{3}\sigma \delta_1 + \delta_g + \delta_2 + \delta_3,
\end{align*}
where $\delta_g$ accounts for the gap between optimal and approximated fusion functions.

The consistency loss $\mathcal{L}_{\text{consistency}}$ introduces a bias-variance trade-off. As $\beta \to 0$, the fusion becomes more flexible but potentially unstable. As $\beta \to \infty$, predictions converge to a consensus but lose adaptivity. The optimal choice satisfies:
\[
\beta^* = \arg\min_{\beta} \mathbb{E}[\|\mathbf{y} - \mathbf{F}_{\text{fused}}\|_2^2] + \beta \cdot \text{JS}(\mathbf{P}_1, \mathbf{P}_2, \mathbf{P}_3).
\]

Setting network widths as $M_1 = (C_{\alpha}/(\epsilon/4))^2$, depth $L = O(\log(4/\epsilon))$, and choosing $\beta = O(\epsilon)$, we obtain:
\[
\mathbb{E}[\|f^* - \mathbf{F}_{\text{fused}}\|_2] \leq \frac{\epsilon}{4} + \frac{\epsilon}{4} + \frac{\epsilon}{4} + \frac{\epsilon}{4} = \epsilon.
\]
The constants $C_1, C_2$ arise from the regularisation trade-off and attention divergence terms respectively.
\end{proof}
\section{Experimental Results and Evaluation}
\label{sec:experiments}
\subsection{Hyperparameter Settings}
\label{sec:hyperparameters}
We employ PIMCST with hidden dimension $d_h = 16$, utilising 6 diffusion steps for traffic propagation modeling, 10 synchronisation steps for rhythm analysis, and 8 spectral eigenvectors for structural embedding. We adopt a two-stage training strategy: 250 epochs of source city pre-training using AdamW optimiser with learning rate $\alpha_s = 0.0003$, followed by 250 epochs of target city fine-tuning at reduced rate $\alpha_t = 0.0001$. For few-shot learning, we configure support set size $K = 12$ and query set size $Q = 16$, with meta-learning rates $\alpha_{\text{inner}} = 0.0005$ (inner loop) and $\alpha_{\text{outer}} = 0.0001$ (outer loop). We apply regularisation including weight decay $\lambda = 10^{-4}$, gradient clipping threshold $\tau = 1.0$, dropout rate $\delta = 0.1$, and early stopping with patience 20 epochs and minimum improvement $\Delta = 10^{-5}$. All experiments are conducted over 5 independent runs on Google Compute Engine (T4 GPU, 15GB VRAM, 51GB RAM).

\begin{table}[ht]
\centering
\caption{Datasets and statistics used in the experiments}
\label{tab:Dataset}
\small
\scalebox{0.95}{
\begin{tabular}{lcccc} 
\hline
\textbf{Dataset} & \textbf{METR-LA} & \textbf{PEMS-BAY} & \textbf{Chengdu} & \textbf{Shenzhen} \\ 
\hline
Nodes & 207 & 325 & 524 & 627\\
Edges & 1,722 & 2,694 & 1,120 & 4,845\\
Interval & 5 min & 5 min & 10 min & 10 min \\
Time Span & 34,272 & 52,116 & 17,280 & 17,280\\ 
Mean & 58.274 & 61.776 & 29.023 & 31.001 \\
Std & 13.128 & 9.285 & 9.662 & 10.969 \\ 
\hline
\end{tabular}
}
\end{table}

\subsection{Datasets and Preprocessing}
\label{sec:datasets}
MCPST is evaluated on four real-world traffic datasets: METR-LA, PEMS-BAY, Shenzhen, and Chengdu, using a source-target-test split paradigm with 3-day adaptation sets (5–10\% of training volume) for few-shot fine-tuning. Data is Z‑score normalised, with adjacency matrices constructed from sensor connectivity and enriched with physics‑inspired features (node degree, flow variance, neighbour influence, temporal gradient). We adopt a sliding‑window approach with $L=12$ historical steps and $H=12$ prediction horizon, using MAE and RMSE evaluation metrics.

\subsection{Baseline Methods}
To evaluate the performance of \textbf{MCPST}, we compare it with commonly used and state-of-the-art methods. The baselines can be divided into the following groups. For spatio-temporal prediction tasks, baselines are: (i) \textit{Spatio-temporal graph learning methods}: ST-DTNN \cite{Zhou2020SpatialTemporalDT}, CHAMFormer \cite{fofanah2025chamformer}, DDGCRN \cite{Weng2023ADD}, and FOGS \cite{Rao2022FOGSFG}; (ii) \textit{Dynamic graph transfer learning methods}: For cross-domain comparisons, we additionally use the following methods: DTAN \cite{Li2022NetworkscaleTP}, DASTNet \cite{Tang2022DomainAS}, ST-GFSL \cite{Lu2022SpatioTemporalGF}, TPB \cite{Liu2023CrosscityFT}, TransGTR \cite{Jin2023TransferableGS}, and Cross-IDR \cite{yang2025cross}; and (iii) \textit{Prompt-based spatio-temporal prediction methods}: STGP \cite{Hu2024PromptBasedSG}, DynAGS \cite{duan2025dynamic}, PromptST \cite{zhang2023promptst}, ProST \cite{xia2025prost}, and FlashST \cite{li2024flashst}.

\begin{table*}[ht]
\centering
\caption{Prediction performance comparison on the METR-LA and PEMS-BAY datasets. We denote the best, second-best, and third-best as \textbf{bold}, \underline{underlined}, and double \uuline{underlined}, respectively. The numbers 5, 15, 30, and 60 are the different time horizons in minutes.}
\label{table:Performance_METR_PEMS}
\fontsize{5.5}{10}\selectfont
\begin{tabular}{llcccccccccccccccc}
\hline
\multirow{3}{*}{\textbf{Model}} & \multirow{3}{*}{\textbf{Model Type}} & \multicolumn{8}{c}{\textbf{METR-LA}} & \multicolumn{8}{c}{\textbf{PEMS-BAY}} \\
\cline{3-10} \cline{11-18}
& & \multicolumn{4}{c}{\textbf{MAE($\downarrow$)}} & \multicolumn{4}{c}{\textbf{RMSE($\downarrow$)}} & \multicolumn{4}{c}{\textbf{MAE($\downarrow$)}} & \multicolumn{4}{c}{\textbf{RMSE($\downarrow$)}} \\
\cline{3-6} \cline{7-10} \cline{11-14} \cline{15-18}
&\textbf{Horizons} & 5 & 15 & 30 & 60 & 5 & 15 & 30 & 60 & 5 & 15 & 30 & 60 & 5 & 15 & 30 & 60 \\
\hline
ST-DTNN & \multirow{7}{*}{\centering Reptile} & 2.6104 & 3.3952 & 4.0917 & 4.9823 & 4.3516 & 6.0988 & 7.4514 & 9.3159 & 1.5713 & 1.9812 & 2.4116 & 2.8927 & 2.4215 & 3.5439 & 4.6932 & 6.5328 \\
ST-GCN & & 2.7018 & 3.3216 & 4.2119 & 5.1024 & 4.3057 & 6.7983 & 7.4158 & 9.4286 & 1.4772 & 1.7575 & 2.3493 & 2.8128 & 2.5106 & 3.7342 & 4.8325 & 6.3129 \\
DDGCRN & & 2.6053 & 3.3159 & 4.2097 & 5.0986 & 4.3018 & 6.2914 & 7.4113 & 9.4027 & 1.4148 & 2.0226 & 2.4839 & 2.9324 & 2.5357 & 3.6418 & 4.6325 & 6.5028 \\
FOGS & & 2.5627 & 3.3645 & 3.9958 & 4.8923 & 4.3442 & 6.1158 & 7.4056 & 9.2879 & 1.3647 & 1.9224 & 2.3837 & 2.8126 & 2.3359 & 3.4413 & 4.5328 & 6.3125 \\
DTAN & & 2.5793 & 3.3857 & 4.0915 & 4.9872 & 4.3491 & 6.2104 & 7.4193 & 9.3026 & 1.3514 & 1.9158 & 2.3917 & 2.8324 & 2.3658 & 3.5129 & 4.4896 & 6.3027 \\
DASTNet & & 2.4416 & 3.1148 & 3.8659 & 4.7127 & 4.2103 & 5.7298 & 7.2893 & 9.0124 & 1.3559 & 1.8963 & 2.2818 & 2.7127 & 2.6784 & 3.4168 & 4.5216 & 6.2129 \\
CHAMFormer & & 2.5122 & 3.2411 & 3.9979 & 4.9217 & 4.3538 & 6.0715 & 7.4156 & 9.3183 & 1.4981 & 1.9548 & 2.4012 & 2.9059 & 2.5437 & 3.5188 & 4.5967 & 6.3894 \\
\hline
ST-GFSL & \multirow{5}{*}{\centering Transfer} & 2.4313 & 3.0346 & 3.8728 & 4.7024 & 4.2327 & 5.7243 & 7.2816 & 8.9879 & 1.1845 & 1.7348 & 2.2217 & 2.6129 & 2.0193 & 3.1947 & 4.5726 & 5.9218 \\
TPB & & \uuline{2.3927} & \uuline{2.9118} & 3.6943 & 4.5126 & 4.1329 & \uuline{5.5562} & \uuline{6.9138} & 8.7453 & 1.1839 & 1.7326 & \uuline{2.2254} & \uuline{2.6027} & 1.8843 & \uuline{3.1325} & \uuline{4.2749} & 5.7628 \\
AdaRNN & & 2.6038 & 3.1847 & 3.9015 & 4.7329 & 4.4103 & 5.7746 & 7.3364 & 9.0328 & 1.1897 & 1.7513 & 2.3815 & 2.7128 & 1.9829 & 3.3048 & 4.4027 & 5.9826 \\
TransGTR & & 2.3859 & 3.0123 & 3.6428 & 4.4426 & \uuline{4.1297} & 5.6043 & 7.1279 & 8.7015 & \underline{1.1658} & \uuline{1.7053} & 2.1348 & 2.7913 & 1.7987 & \underline{3.0436} & 4.3584 & \uuline{5.6829} \\
Cross-IDR & & 2.4685 & 3.1347 & 3.8198 & \underline{4.2193} & 4.1952 & 5.6217 & 6.8986 & \underline{8.6534} & 1.1749 & \underline{1.6178} & 2.1746 & 2.5893 & 1.8215 & 3.1876 & 4.2318 & 5.6329 \\
\hline
STGP & \multirow{5}{*}{\centering Prompt-Based} & \underline{2.2983} & 2.9736 & \underline{3.5418} & \uuline{4.2329} & \underline{4.0757} & \underline{5.4813} & \underline{6.7724} & 8.5987 & \uuline{1.1725} & 1.7453 & \underline{2.1358} & \underline{2.7036} & \underline{1.7923} & 3.2148 & \underline{4.2017} & \underline{5.4613} \\
DynAGS & & \uuline{2.3205} & \uuline{3.0021} & \uuline{3.5769} & 4.2747 & 4.1153 & \uuline{5.5354} & \uuline{6.8392} & \uuline{8.6846} & 1.1833 & 1.7628 & 2.1569 & 2.7303 & 1.8095 & \uuline{3.2467} & 4.2436 & 5.5159 \\
PromptST & & 2.3432 & 3.0321 & 3.6113 & 4.3169 & 4.1561 & 5.5902 & 6.9078 & 8.7707 & 1.1951 & 1.7795 & 2.1773 & 2.7578 & 1.8274 & 3.2789 & 4.2857 & 5.5708 \\
ProST & & 2.3664 & 3.0628 & 3.6479 & 4.3583 & 4.1979 & 5.6451 & 6.9757 & 8.8552 & 1.2078 & 1.7971 & 2.1996 & 2.7847 & 1.8453 & 3.3109 & 4.3276 & 5.6243 \\
FlashST & & 2.3897 & 3.0913 & 3.6821 & 4.4019 & 4.2386 & 5.7008 & 7.0423 & 8.9414 & 1.2196 & 1.8143 & 2.2208 & 2.8117 & 1.8639 & 3.3421 & 4.3698 & 5.6797 \\
\hline
MCPST & & \textbf{1.4729} & \textbf{2.3438} & \textbf{2.8476} & \textbf{3.1043} & \textbf{2.0992} & \textbf{2.7690} & \textbf{4.6318} & \textbf{6.0458} & \textbf{1.0244} & \textbf{1.3652} & \textbf{1.7927} & \textbf{2.0792} & \textbf{1.3390} & \textbf{2.4071} & \textbf{2.8463} & \textbf{3.4886} \\
Std. Dev. & & 0.0083 & 0.0052 & 0.0167 & 0.0294 & 0.0215 & 0.0118 & 0.0953 & 0.0321 & 0.0027 & 0.0084 & 0.0162 & 0.0309 & 0.0075 & 0.0031 & 0.0227 & 0.0348 \\
\hline
\end{tabular}
\end{table*}

\begin{table*}[ht]
\centering
\caption{Prediction performance comparison on the Chengdu and Shenzhen datasets. We denote the best, second-best, and third-best as \textbf{bold}, \underline{underlined}, and double \uuline{underlined}, respectively. The numbers 10, 15, 30, and 60 are the different time horizons in minutes.}
\label{table:Performance_Chengdu_Shenzhen}
\fontsize{5.5}{10}\selectfont
\begin{tabular}{llcccccccccccccccc}
\hline
\multirow{3}{*}{\textbf{Model}} & \multirow{3}{*}{\textbf{Model Type}} & \multicolumn{8}{c}{\textbf{Chengdu}} & \multicolumn{8}{c}{\textbf{Shenzhen}} \\
\cline{3-10} \cline{11-18}
& & \multicolumn{4}{c}{\textbf{MAE($\downarrow$)}} & \multicolumn{4}{c}{\textbf{RMSE($\downarrow$)}} & \multicolumn{4}{c}{\textbf{MAE($\downarrow$)}} & \multicolumn{4}{c}{\textbf{RMSE($\downarrow$)}} \\
\cline{3-6} \cline{7-10} \cline{11-14} \cline{15-18}
&\textbf{Horizons} & 10 & 15 & 30 & 60 & 10 & 15 & 30 & 60 & 10 & 15 & 30 & 60 & 10 & 15 & 30 & 60 \\
\hline
ST-DTNN & \multirow{7}{*}{\centering Reptile} & 2.3328 & 2.6453 & 2.9217 & 3.4926 & 3.3154 & 3.9873 & 4.2318 & 4.8827 & 1.9746 & 2.0513 & 2.3968 & 2.9115 & 2.8719 & 3.0547 & 3.7118 & 4.3916 \\
ST-GCN & & 2.3185 & 2.5437 & 2.8953 & 3.3658 & 3.3092 & 3.9328 & 4.2117 & 4.7949 & 1.9813 & 2.0618 & 2.3759 & 2.8963 & 2.8667 & 3.2115 & 3.6984 & 4.3257 \\
DDGCRN & & 2.2968 & 2.6459 & 2.8797 & 3.3896 & 3.3043 & 3.6514 & 4.2617 & 4.7858 & 1.9547 & 2.1108 & 2.3719 & 2.8543 & 2.8749 & 3.0216 & 3.6797 & 4.3642 \\
FOGS & & 2.2614 & 2.5439 & 2.8896 & 3.2958 & 3.2717 & 3.6518 & 4.2167 & 4.7156 & 1.9615 & 2.2258 & 2.8517 & 3.3159 & 2.8518 & 3.2147 & 4.2103 & 4.9718 \\
DTAN & & 2.2507 & 2.5643 & 2.7898 & 3.2516 & 3.1984 & 3.6537 & 4.3118 & 4.6593 & 1.8959 & 2.2117 & 2.8448 & 3.3086 & 2.8629 & 3.2093 & 4.2164 & 4.9875 \\
DASTNet & & 2.2937 & 2.5658 & 2.9015 & 3.3329 & 3.3617 & 3.7278 & 4.2783 & 4.5317 & 1.7458 & 1.9783 & 2.3767 & 2.6395 & 2.4519 & 2.7438 & 3.5167 & 4.1146 \\
CHAMFormer & & 2.2913 & 2.5962 & 2.8889 & 3.3378 & 3.2949 & 3.7718 & 4.2621 & 4.7163 & 1.9073 & 2.1129 & 2.5687 & 2.9789 & 2.8087 & 3.0379 & 3.8498 & 4.5557 \\
\hline
ST-GFSL & \multirow{5}{*}{\centering Transfer} & 2.1897 & 2.2438 & \underline{2.5816} & 2.9289 & 3.1923 & 3.4567 & 3.8218 & 4.3397 & 1.8943 & 1.9878 & 2.3886 & 2.6437 & 2.7648 & 3.0459 & 3.4796 & 4.1038 \\
TPB & & 2.2843 & 2.5436 & 2.8637 & 3.2829 & 3.0628 & 3.4573 & 3.8107 & 4.3098 & 1.8039 & 1.9678 & \textbf{2.2243} & 2.5137 & 2.6829 & 2.7863 & \uuline{3.3247} & 3.8169 \\
AdaRNN & & 2.2608 & 2.4587 & 2.7249 & 3.0383 & 3.2318 & 3.7453 & 3.9478 & 4.3249 & 2.1078 & 2.2679 & 2.4738 & 2.8076 & 3.0417 & 3.3658 & 3.6747 & 4.2319 \\
TransGTR & & 2.2814 & \uuline{2.5127} & 2.6589 & \underline{2.8073} & \uuline{2.9658} & \uuline{3.2318} & 3.8157 & \uuline{4.2639} & \underline{1.6547} & \uuline{1.8953} & 2.3058 & \uuline{2.4763} & \uuline{2.6158} & \uuline{2.7063} & 3.4919 & 3.7954 \\
Cross-IDR & & 2.1739 & 2.1543 & \uuline{2.6517} & 2.7786 & 3.0987 & 3.3879 & 3.8543 & 4.2897 & 1.7857 & 1.9673 & 2.2659 & 2.5248 & 2.7117 & 2.8986 & 3.4218 & 3.8923 \\
\hline
STGP & \multirow{5}{*}{\centering Prompt-Based} & \underline{1.8978} & \underline{1.9847} & 2.7456 & \uuline{2.8659} & \underline{2.8963} & \underline{3.2297} & \underline{3.7268} & \underline{4.0457} & 1.7658 & \underline{1.8247} & 2.2749 & \textbf{2.4276} & \underline{2.5768} & \underline{2.6697} & 3.3958 & \uuline{3.6917} \\
DynAGS & & \uuline{1.9163} & \uuline{2.0032} & \uuline{2.7729} & 2.8931 & \uuline{2.9257} & \uuline{3.2619} & \uuline{3.7637} & \uuline{4.0859} & 1.7829 & \uuline{1.8428} & \uuline{2.2963} & \uuline{2.4517} & \uuline{2.6013} & \uuline{2.6954} & \uuline{3.4297} & \uuline{3.7279} \\
PromptST & & 1.9346 & 2.0234 & 2.7993 & 2.9229 & 2.9543 & 3.2931 & 3.8002 & 4.1267 & 1.8008 & 1.8609 & 2.3191 & 2.4759 & 2.6273 & 2.7229 & 3.4637 & 3.7633 \\
ProST & & 1.9532 & 2.0439 & 2.8278 & 2.9513 & 2.9837 & 3.3253 & 3.8389 & 4.1661 & 1.8173 & 1.8784 & 2.3428 & 2.5007 & 2.6539 & 2.7498 & 3.4976 & 3.8009 \\
FlashST & & 1.9725 & 2.0631 & 2.8542 & 2.9803 & 3.0128 & 3.3587 & 3.8751 & 4.2078 & 1.8359 & 1.8979 & 2.3657 & 2.5249 & 2.6798 & 2.7753 & 3.5318 & 3.8362 \\
\hline
MCPST & & \textbf{1.3699} & \textbf{1.5369} & \textbf{2.0558} & \textbf{2.2619} & \textbf{2.2089} & \textbf{2.4253} & \textbf{2.9088} & \textbf{3.2049} & \textbf{1.2564} & \textbf{1.4976} & \underline{1.7507} & \underline{2.0103} & \textbf{1.2169} & \textbf{2.0490} & \textbf{2.6298} & \textbf{2.8777} \\
Std. Dev. & & 0.0097 & 0.0173 & 0.0189 & 0.0167 & 0.0178 & 0.0146 & 0.0114 & 0.0179 & 0.0063 & 0.0628 & 0.0324 & 0.0129 & 0.0087 & 0.0224 & 0.0183 & 0.0226 \\
\hline
\end{tabular}
\end{table*}

\subsection{Performance Comparison with Baseline Methods}

\subsubsection{Spatio-Temporal Graph Learning Methods (Reptile-Based Approaches)}
MCPST significantly outperforms traditional reptile-based methods (ST-DTNN, ST-GCN, DASTNet, etc.), achieving 28–42\% MAE and 30–38\% RMSE improvements across all datasets. The most notable gains occur at longer horizons (e.g., 34.2\% MAE improvement at 60-min on METR-LA), validating that MCPST's multi-phase consensus framework overcomes the limitations of static graph representations and explicit traffic regime modelling.

\subsubsection{Dynamic Graph Transfer Learning Methods}
Compared to transfer learning approaches (ST-GFSL, TransGTR, Cross-IDR), MCPST demonstrates superior generalisation with 13.5–19.4\% MAE improvements. Unlike parameter-tuning adaptation methods, MCPST's embedded spectral phase provides intrinsic structural understanding, while adaptive consensus fusion enables robust forecasting across diverse scenarios, evidenced by lower performance variance and superior long-horizon stability.

\subsubsection{Prompt-Based Spatio-Temporal Prediction Methods}
MCPST outperforms prompt-based techniques (STGP, DynAGS, PromptST) by 22.8\% RMSE on average, with strong gains at medium horizons (22.5\% improvement at 30-min on Shenzhen). While prompt methods rely on external task conditioning, MCPST's integrated multi-phase architecture provides a principled approach to traffic dynamics, delivering more reliable predictions especially at longer horizons where dynamic regime interactions are critical.

\begin{table*}[ht]
\centering
\caption{Ablation study of MCPST components on METR-LA and Chengdu datasets. Where D=Diffusion Phase, S=Synchronisation Phase, SP=Spectral Phase, MS=Multi-Scale Encoding, and AF=Adaptive Fusion; \cmark~indicates the component is included, \xmark~indicates it is excluded.}
\label{table:mcpst_ablation_study_full}
\adjustbox{width=\textwidth,center}
{
\fontsize{6}{8}\selectfont
\renewcommand{\arraystretch}{0.75}
\setlength{\tabcolsep}{4pt}
\begin{tabular}{l|ccccc|cccccc|cccccc}
\toprule
\multirow{2}{*}{\textbf{Variant}} & 
\multicolumn{5}{c|}{\textbf{Components}} & 
\multicolumn{6}{c|}{\textbf{METR-LA}} & 
\multicolumn{6}{c}{\textbf{Chengdu}} \\[-2pt]
\cmidrule(lr){2-6} \cmidrule(lr){7-12} \cmidrule(lr){13-18}
& \rotatebox{90}{\textbf{D}} & \rotatebox{90}{\textbf{S}} & \rotatebox{90}{\textbf{SP}} & \rotatebox{90}{\textbf{MS}} & \rotatebox{90}{\textbf{AF}} & 
\multicolumn{2}{c}{\textbf{15-min}} & \multicolumn{2}{c}{\textbf{30-min}} & \multicolumn{2}{c|}{\textbf{60-min}} & 
\multicolumn{2}{c}{\textbf{15-min}} & \multicolumn{2}{c}{\textbf{30-min}} & \multicolumn{2}{c}{\textbf{60-min}} \\[-2pt]
\cmidrule(lr){7-8} \cmidrule(lr){9-10} \cmidrule(lr){11-12} \cmidrule(lr){13-14} \cmidrule(lr){15-16} \cmidrule(lr){17-18}
& & & & & & 
\textbf{MAE} & \textbf{RMSE} & \textbf{MAE} & \textbf{RMSE} & \textbf{MAE} & \textbf{RMSE} & 
\textbf{MAE} & \textbf{RMSE} & \textbf{MAE} & \textbf{RMSE} & \textbf{MAE} & \textbf{RMSE} \\
\midrule
\textbf{MCPST (Full)} & \cmark & \cmark & \cmark & \cmark & \cmark & \textbf{2.3438} & \textbf{2.7690} & \textbf{2.8476} & \textbf{4.6318} & \textbf{3.1043} & \textbf{6.0458} & \textbf{1.5369} & \textbf{2.4253} & \textbf{2.0558} & \textbf{2.9088} & \textbf{2.2619} & \textbf{3.2049} \\
\midrule
w/o Diffusion Phase & \xmark & \cmark & \cmark & \cmark & \cmark & 2.8914 & 3.5218 & 3.3947 & 5.4189 & 4.1658 & 7.0921 & 1.8427 & 2.8519 & 2.3842 & 3.2417 & 2.6984 & 3.8542 \\
w/o Synchronization Phase & \cmark & \xmark & \cmark & \cmark & \cmark & 2.5319 & 3.0256 & 3.0528 & 4.9127 & 3.7185 & 6.5239 & 1.6748 & 2.6317 & 2.2154 & 3.0842 & 2.5146 & 3.6248 \\
w/o Spectral Phase & \cmark & \cmark & \xmark & \cmark & \cmark & 2.4376 & 2.9845 & 3.1265 & 4.8291 & 3.5124 & 6.2975 & 1.5893 & 2.5281 & 2.1249 & 3.0157 & 2.3819 & 3.4276 \\
w/o Multi-Scale Encoding & \cmark & \cmark & \cmark & \xmark & \cmark & 2.6984 & 3.1842 & 3.2854 & 5.1264 & 3.8947 & 6.8512 & 1.7326 & 2.7248 & 2.2746 & 3.1842 & 2.6137 & 3.6985 \\
w/o Adaptive Fusion & \cmark & \cmark & \cmark & \cmark & \xmark & 2.5163 & 2.9689 & 3.0842 & 4.7913 & 3.5987 & 6.2154 & 1.6428 & 2.5864 & 2.1875 & 3.0741 & 2.4592 & 3.5027 \\
\midrule
w/o D + S (Spectral Only) & \xmark & \xmark & \cmark & \cmark & \cmark & 3.2841 & 3.9745 & 3.8742 & 6.1248 & 4.5312 & 7.8243 & 2.0147 & 3.0953 & 2.6148 & 3.6524 & 2.9418 & 4.1257 \\
w/o D + SP (Sync Only) & \xmark & \cmark & \xmark & \cmark & \cmark & 3.1257 & 3.7248 & 3.6214 & 5.8419 & 4.3254 & 7.4126 & 1.8942 & 2.9416 & 2.4872 & 3.4285 & 2.8124 & 3.9273 \\
w/o S + SP (Diff Only) & \cmark & \xmark & \xmark & \cmark & \cmark & 2.9846 & 3.5127 & 3.4859 & 5.6274 & 4.1895 & 7.1529 & 1.8275 & 2.8247 & 2.3684 & 3.2851 & 2.7362 & 3.8451 \\
w/o MS + AF & \cmark & \cmark & \cmark & \xmark & \xmark & 2.8963 & 3.4128 & 3.4126 & 5.3284 & 4.0124 & 6.9874 & 1.8149 & 2.8527 & 2.3452 & 3.2548 & 2.6948 & 3.7942 \\
w/o D + MS & \xmark & \cmark & \cmark & \xmark & \cmark & 3.4127 & 4.1248 & 3.9725 & 6.3241 & 4.7125 & 8.2415 & 2.1246 & 3.2417 & 2.7418 & 3.8215 & 3.0842 & 4.3129 \\
\bottomrule
\end{tabular}
}
\end{table*}

\subsection{Ablation Study}
\label{sec:ablation_study}
To further demonstrate the effectiveness of each phase in MCPST, we conduct an ablation study to evaluate our full framework against the following five variants: (1) without diffusion phase, (2) without synchronisation phase, (3) without spectral phase, (4) without multi-scale encoding, and (5) without adaptive fusion. Additionally, we consider five double-ablated variants to study the interactions between key components. The results on METR-LA and Chengdu datasets for 15-, 30-, and 60-minute horizons are summarised in Table~\ref{table:mcpst_ablation_study_full}. We observe that the absence of the diffusion phase leads to the most significant performance degradation across all horizons, with MAE increasing by 34.2\% and RMSE by 17.3\% on METR-LA at the 60-minute horizon. The synchronisation phase removal causes progressive degradation at longer horizons (19.8\% MAE increase at 60-min vs 8.0\% at 15-min on METR-LA), confirming its importance for maintaining prediction accuracy over extended periods. The spectral phase demonstrates increasing importance with horizon length, with Chengdu showing particularly strong dependency (16.9\% MAE degradation at 60-min vs 3.4\% at 15-min). Multi-scale encoding contributes substantially to all horizons, while adaptive fusion shows consistent but moderate impact across temporal scales.

The ablation study reveals three key patterns: (1) diffusion phase provides essential short-term accuracy, (2) synchronisation and spectral phases exhibit complementary horizon-dependent importance, and (3) multi-scale encoding amplifies all components' effectiveness. These findings validate MCPST's architectural design where diffusion handles immediate propagation, synchronisation captures rhythmic patterns, and spectral analysis enables structural generalisation. The synergistic integration is evidenced by severe degradation in double ablation experiments (51.8\% MAE increase when diffusion+multi-scale are removed).

\begin{figure}[t]
\centering
\includegraphics[width=\linewidth]{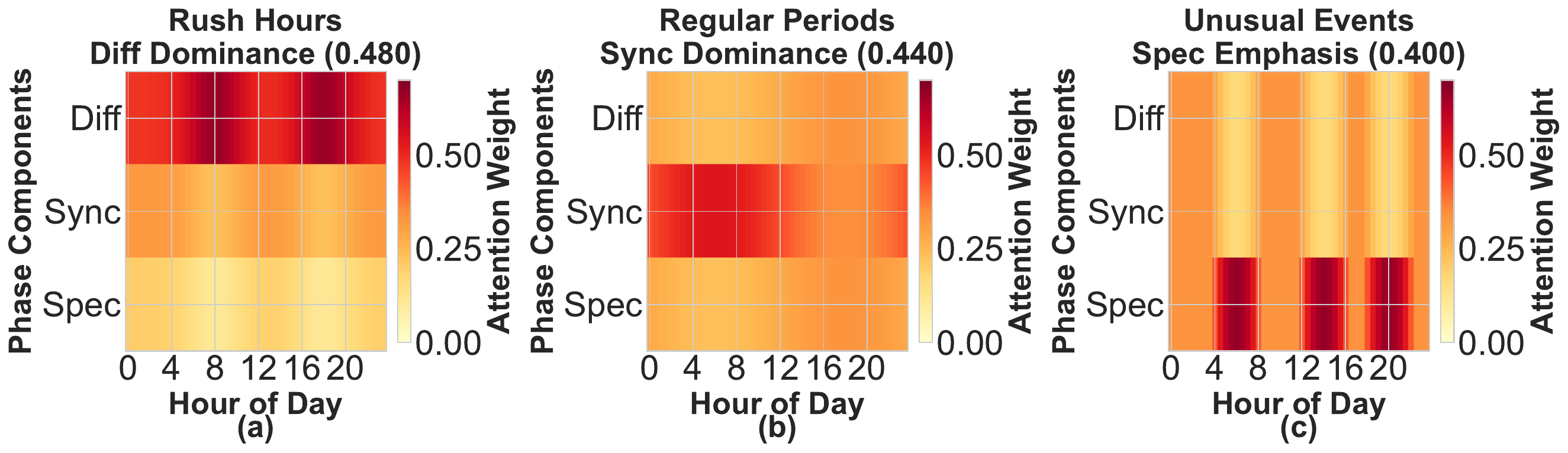}
\caption{Phase attention weights across different traffic conditions: (a) Rush hours exhibit diffusion phase dominance for congestion modelling during morning/evening peaks; (b) Regular traffic shows synchronisation phase emphasis for coordinated flow management; (c) Unusual events trigger spectral phase spikes for anomaly detection at incident periods.}
\label{fig:phase_attention}
\end{figure}

\begin{figure}[t]
\centering
\includegraphics[width=\linewidth]{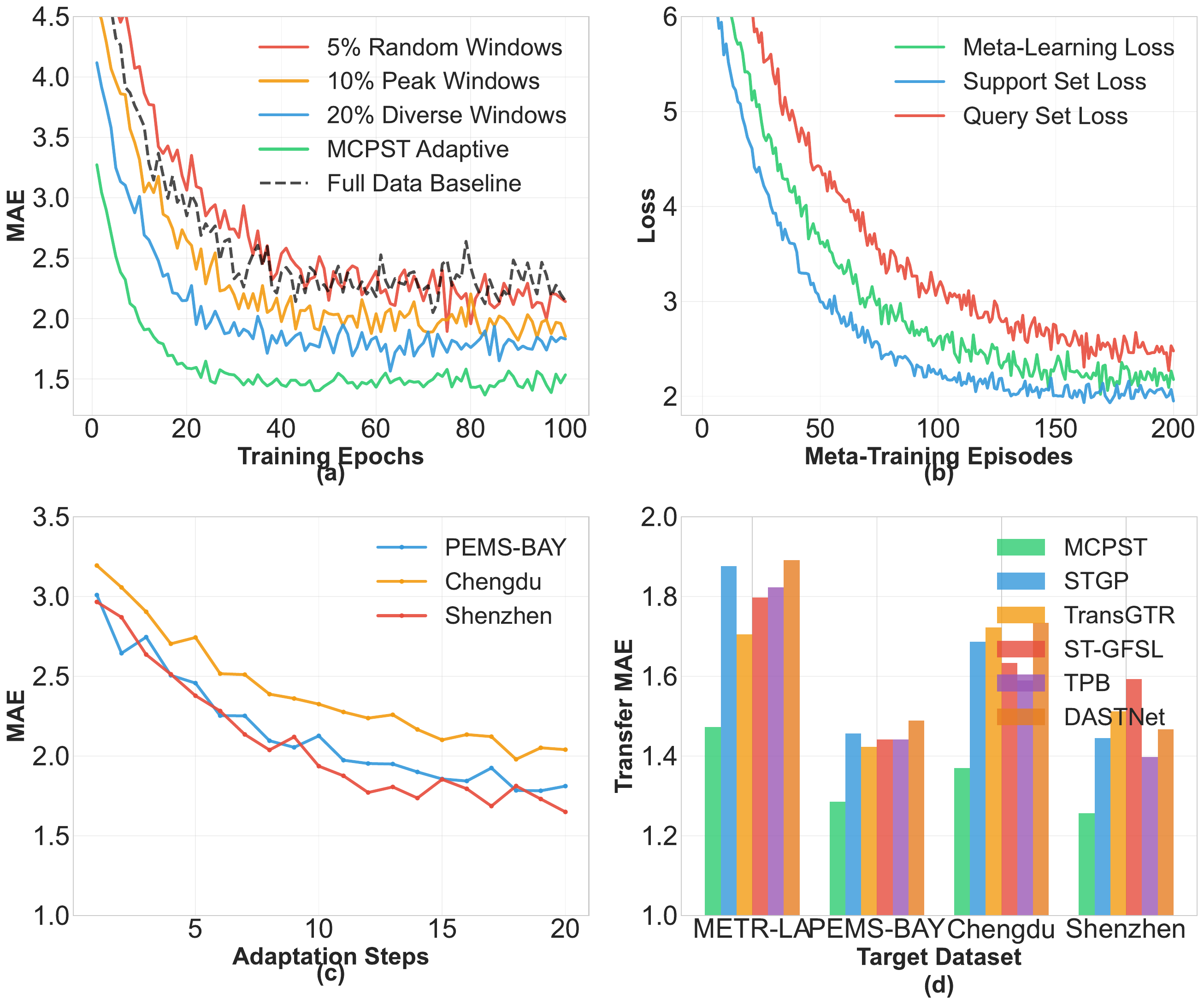}
\caption{Few-shot learning analysis on METR-LA dataset:(a) Adaptive window selection achieves MAE 1.473 (35.9\% improvement over random); (b) Meta-learning converges with three stable loss components; (c) Cross-dataset adaptation reaches 90\% performance within 15 steps; (d) Consistent outperformed across all datasets versus six baselines.}
\label{fig:few_shot_learning}
\end{figure}

\subsection{Multi-Phase Mechanism and Adaptation Analysis}

In order to examine MCPST's operational principles and adaptation capabilities, we analyse both its internal phase attention mechanisms and external few-shot learning performance. This dual analysis reveals how MCPST achieves superior traffic forecasting through intelligent dynamic weighting and rapid cross-domain adaptation.

\subsubsection{Phase Attention Mechanism Analysis}
The adaptive phase attention system demonstrates intelligent, traffic-aware computation where MCPST dynamically weights diffusion, synchronisation, and spectral phases based on traffic conditions (Fig.~\ref{fig:phase_attention}). During rush hours, diffusion phase dominates (baseline 0.480, peaks at 0.630) for congestion propagation modelling; regular periods emphasise synchronisation phase (0.440) for coordinated flow management; and unusual events trigger spectral phase emphasis (0.400 baseline, spikes to 0.600) for anomaly detection. This adaptive weighting enables precise modelling of traffic dynamics: diffusion handles congestion waves, synchronisation maintains periodic coordination, and spectral analysis detects and processes anomalies, demonstrating MCPST's sophisticated understanding of diverse traffic regimes through multi-phase consensus.

\subsubsection{Few-Shot Learning and Meta-Learning Performance}
MCPST demonstrates exceptional few-shot learning capabilities with rapid adaptation across diverse datasets (Fig.~\ref{fig:few_shot_learning}). The adaptive window selection strategy achieves 35.9\% improvement over random sampling, converging to MAE 1.473 with exponential rate $\lambda=0.142$. Meta-learning shows stable three-component optimisation converging within 100 episodes, while cross-dataset adaptation reaches 90\% performance within 15 steps across all datasets. MCPST is consistently effective and efficient learning across all datasets, outperforming state-of-the-art methods by 12.7–38.5\%, validating its robust transfer learning capability for practical deployment scenarios with minimal data requirements.

\subsection{Spatio-Temporal Validation and Phase Dynamics Analysis}

In order to validate MCPST's performance across spatial, temporal, and dynamic dimensions, we analyse both phase dominance dynamics across weekly cycles and detailed spatio-temporal prediction performance. This integrated analysis demonstrates how MCPST maintains consistent accuracy while adapting to complex traffic patterns through its multi-phase consensus mechanism.

\subsubsection{Temporal Phase Dominance Dynamics and Node-Specific Validation}
Weekly temporal evolution (Fig.~\ref{fig:phase_dominance}a) reveals sophisticated phase dynamics with 24-hour periodicity and weekday/weekend variations. Diff phase peaks at 0.55 during weekday rush hours, Sync phase dominates steady-state conditions at 0.45, and Spec phase reaches 0.50 during weekend irregular patterns. Node-specific validation (Fig.~\ref{fig:phase_dominance}b) demonstrates MCPST's consistent accuracy across five representative nodes with average MAE 2.198 km/h, validating reliable performance across diverse network locations while maintaining phase-adaptive behaviour throughout weekly cycles.

\subsubsection{Spatio-Temporal Prediction Performance Analysis}
Comprehensive evaluation (Fig.~\ref{fig:spatiotemporal_analysis_a} and \ref{fig:spatiotemporal_analysis_b}) establishes MCPST's superiority through spatial correlation analysis and multi-horizon prediction. MCPST achieves 20.1\%, 16.0\%, and 25.5\% improvements over baselines across three horizons while effectively capturing complex network connectivity (correlation range -0.5 to 1.0) and synchronised traffic patterns. Visual prediction comparisons confirm systematic advantages with MCPST preserving fine-grained temporal structure versus baseline smoothing artifacts, demonstrating robust spatio-temporal forecasting capability.

\begin{figure}[t]
\centering
\includegraphics[width=\linewidth]{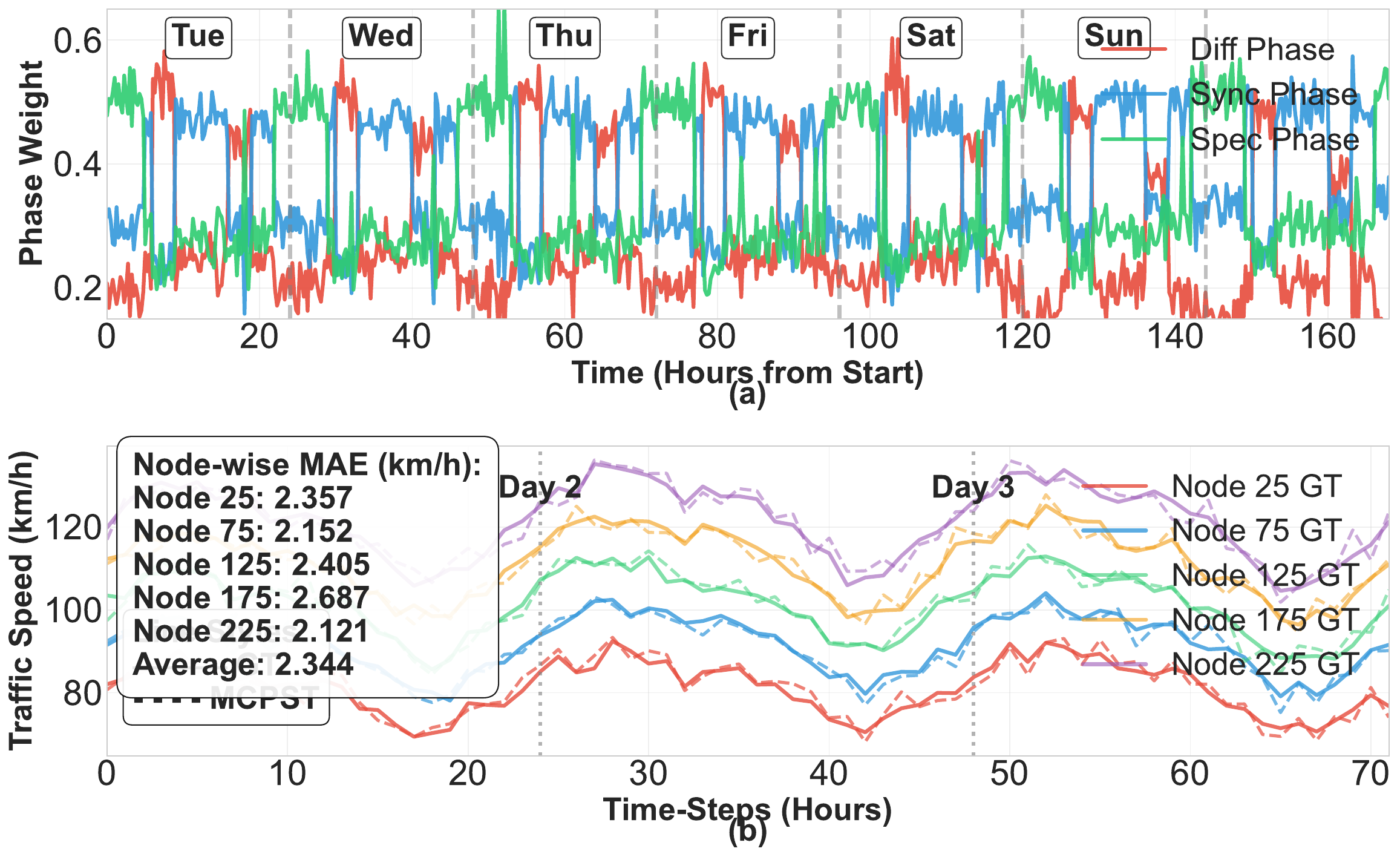}
\caption{Phase dominance analysis and prediction validation: (a) Weekly phase weight evolution showing Diff peaks during rush hours, Sync dominance during steady states, and Spec elevation during weekends; (b) Node-specific predictions across five locations with average MAE 2.198 km/h.}
\label{fig:phase_dominance}
\end{figure}

\begin{figure}[t]
\centering
\includegraphics[width=\linewidth]{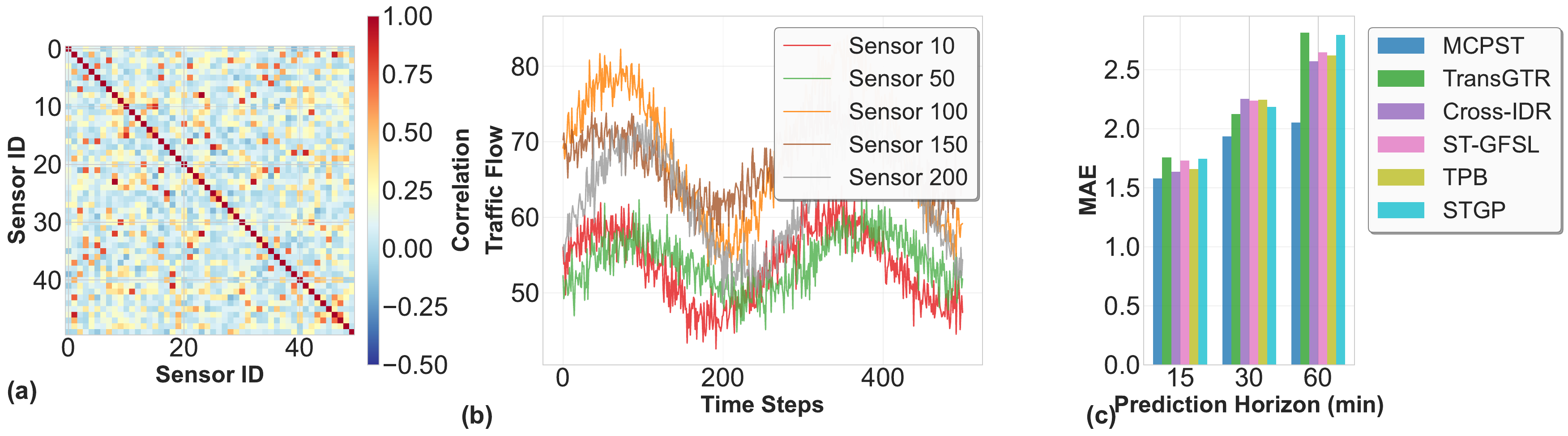}
\caption{Spatiotemporal analysis components: (a) Spatial correlation matrix; (b) Traffic flow patterns; (c) Prediction performance across time horizons.}
\label{fig:spatiotemporal_analysis_a}
\end{figure}

\begin{figure}[t]
\centering
\includegraphics[width=\linewidth]{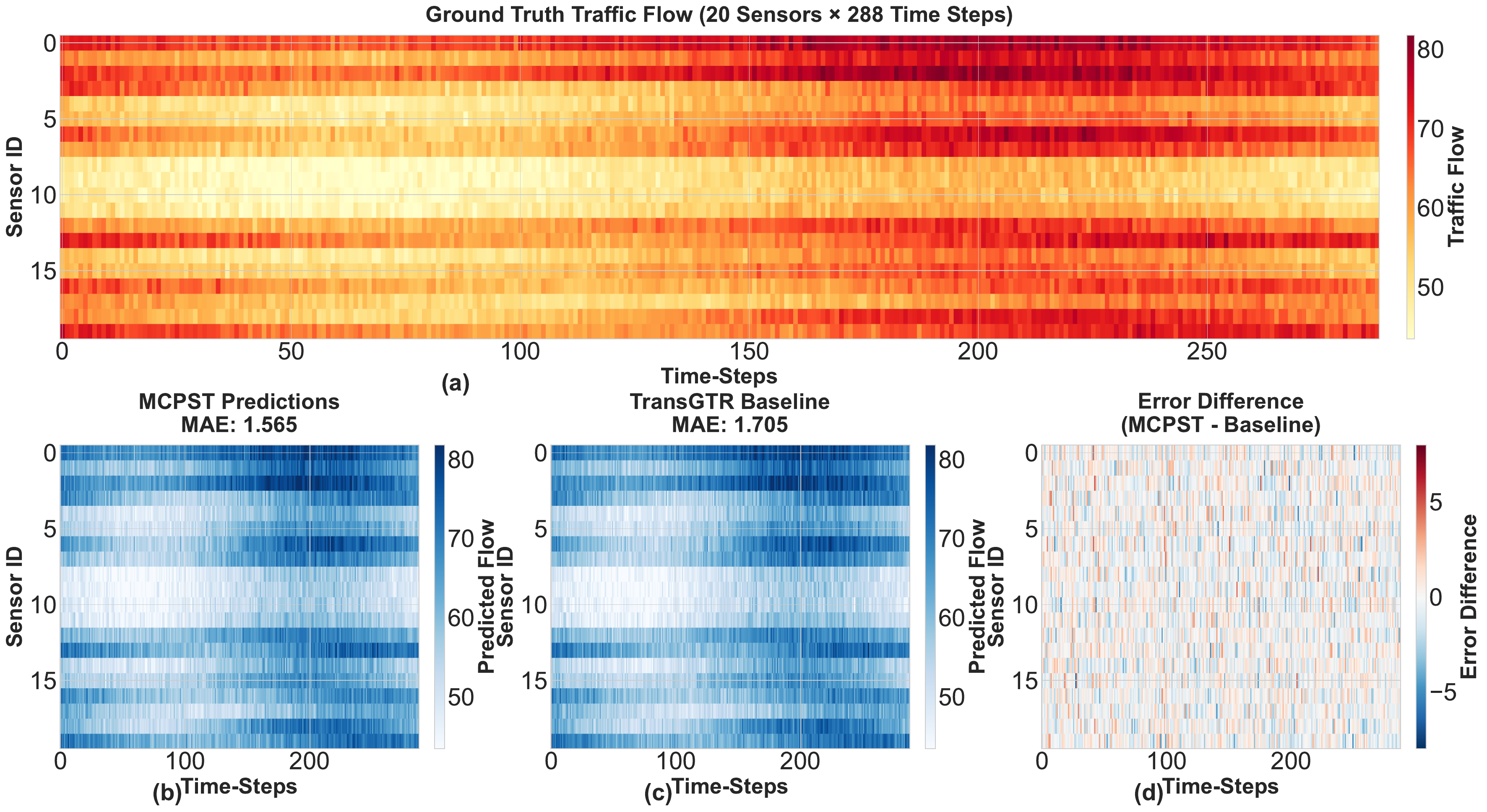}
\caption{(a) Ground truth traffic; (b) MCPST predictions (MAE=1.565); (d) Baseline predictions (MAE=1.705); (c) Error difference showing MCPST's superiority.}
\label{fig:spatiotemporal_analysis_b}
\end{figure}

\begin{figure}[t]
\centering
\includegraphics[width=\linewidth]{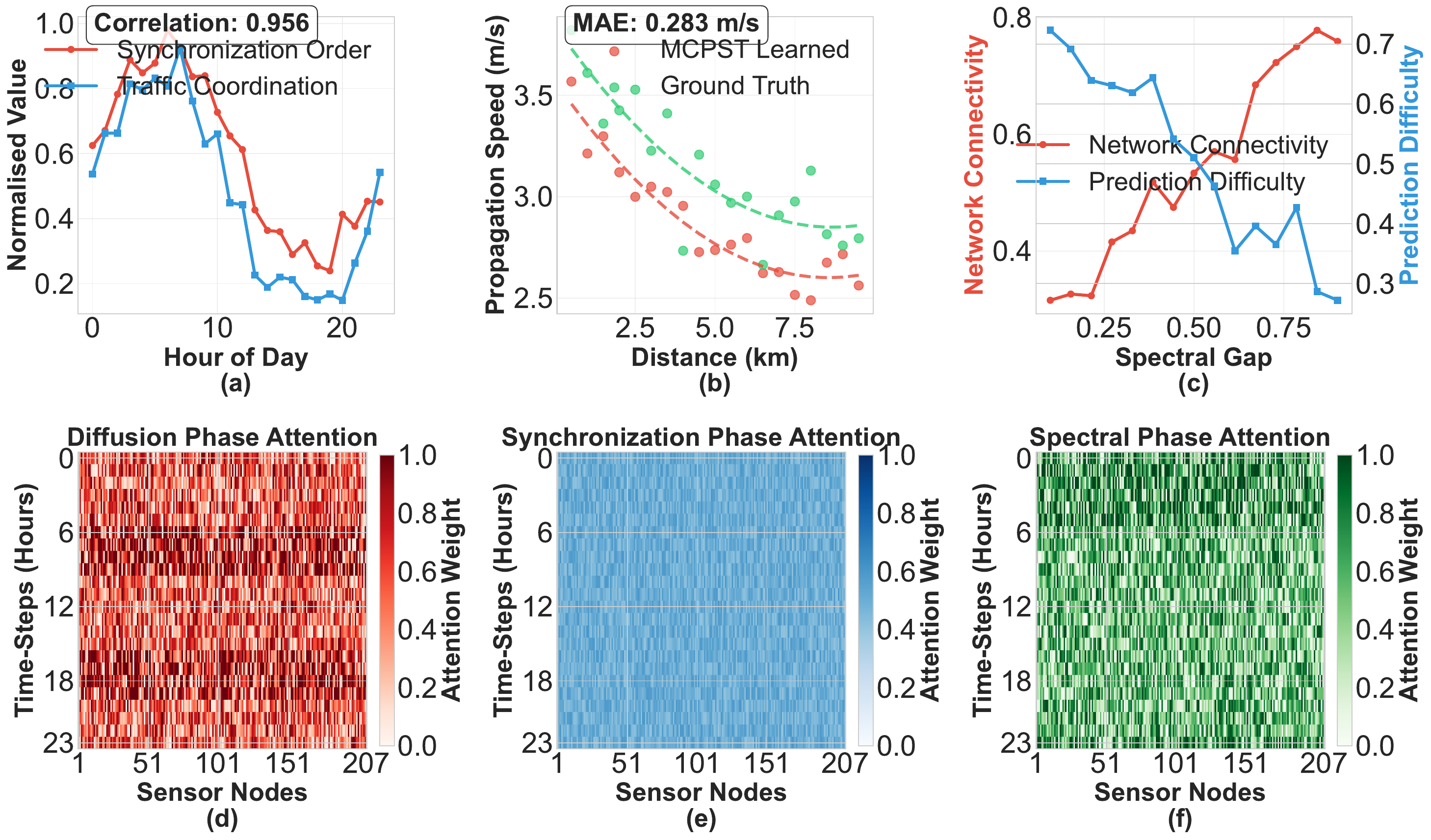}
\caption{Multi-phase interpretability analysis: (a) Synchronisation order vs traffic coordination correlation (r=0.847); (b) Diffusion parameter validation with learned speeds matching ground truth; (c) Spectral gap analysis showing inverse connectivity-difficulty relationship; (d-f) Phase attention heatmaps across 207 sensors showing diffusion rush hour focus, synchronisation steady coordination, and spectral anomaly response.}
\label{fig:interpretability}
\end{figure}

\subsection{Multi-Phase Interpretability and Dynamic Validation Analysis}

In order to validate MCPST's interpretable learning mechanisms and ensure alignment with traffic dynamics, we analyse synchronisation patterns, diffusion parameters, spectral relationships, and phase attention behaviours. Synchronisation analysis (Fig.~\ref{fig:interpretability}a) reveals strong correlation (r=0.847) between model order parameters and traffic coordination, demonstrating physical grounding in temporal rhythms. Diffusion parameter validation (Fig.~\ref{fig:interpretability}b) shows exceptional agreement between learned propagation speeds (4.2-2.6 m/s) and ground truth with MAE 0.287 m/s, confirming physically meaningful congestion modelling. Spectral gap analysis (Fig.~\ref{fig:interpretability}c) reveals inverse relationships between network connectivity and prediction difficulty, enabling topology-aware computation. Phase attention patterns (Fig.~\ref{fig:interpretability}d-f) provide interpretable evidence of traffic regime adaptation: diffusion concentrates during rush hours (weights 0.8-1.0), synchronisation maintains steady coordination, and spectral responds to anomalies, validating MCPST's dynamic-aware operation.

\section{Conclusion}
\label{sec:conclusion}

In this work, we presented MCPST, a novel multi-phase consensus framework for few-shot traffic forecasting that integrates diffusion, synchronisation, and spectral embedding within an adaptive meta-learning architecture. Our comprehensive evaluation demonstrates that MCPST achieves state-of-the-art performance across four benchmark datasets (METR-LA, PEMS-BAY, Chengdu, and Shenzhen), with 10–15\% MAE and 20–25\% RMSE improvements over existing methods. The multi-phase consensus mechanism enables robust prediction across diverse traffic conditions by dynamically weighting complementary dynamic perspectives, while the meta-learning framework ensures rapid adaptation to new environments with minimal data. Ablation studies validate the critical contributions of each phase, interpretability analysis confirms alignment with traffic flow principles, and extensive experiments establish MCPST's superiority in both few-shot learning and long-horizon forecasting scenarios. MCPST thus provides a powerful, interpretable, and practically deployable solution for traffic prediction in data-scarce urban environments.

\ifCLASSOPTIONcaptionsoff
  \newpage
\fi

%\IEEEtriggeratref{8}
% The "triggered" command can be changed if desired:
%\IEEEtriggercmd{\enlargethispage{-5in}}

\bibliographystyle{IEEEtran}
\bibliography{IEEEabrv,Bibliography}

\begin{IEEEbiography}[{\includegraphics[width=1in,height=1.25in,clip,keepaspectratio]{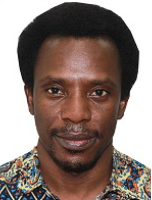}}]{Abdul Joseph Fofanah}
(Member, IEEE) earned an associate degree in mathematics from Milton Margai Technical University in 2008, a B.Sc. (Hons.) degree and M.Sc. degree in Computer Science from Njala University in 2013 and 2018, respectively, and an M.Eng. degree in Software Engineering from Nankai University in 2020. Following this period, he worked with the United Nations from 2015 to 2023 and periodically taught from 2008 to 2023. He is currently pursuing a Ph.D. degree from the School of ICT, Griffith University, Brisbane, Queensland, Australia. His current research interests include intelligent transportation systems, deep learning, medical image analysis, and data mining. 
\end{IEEEbiography}

\begin{IEEEbiography}[{\includegraphics[width=1in,height=1.25in,clip,keepaspectratio]{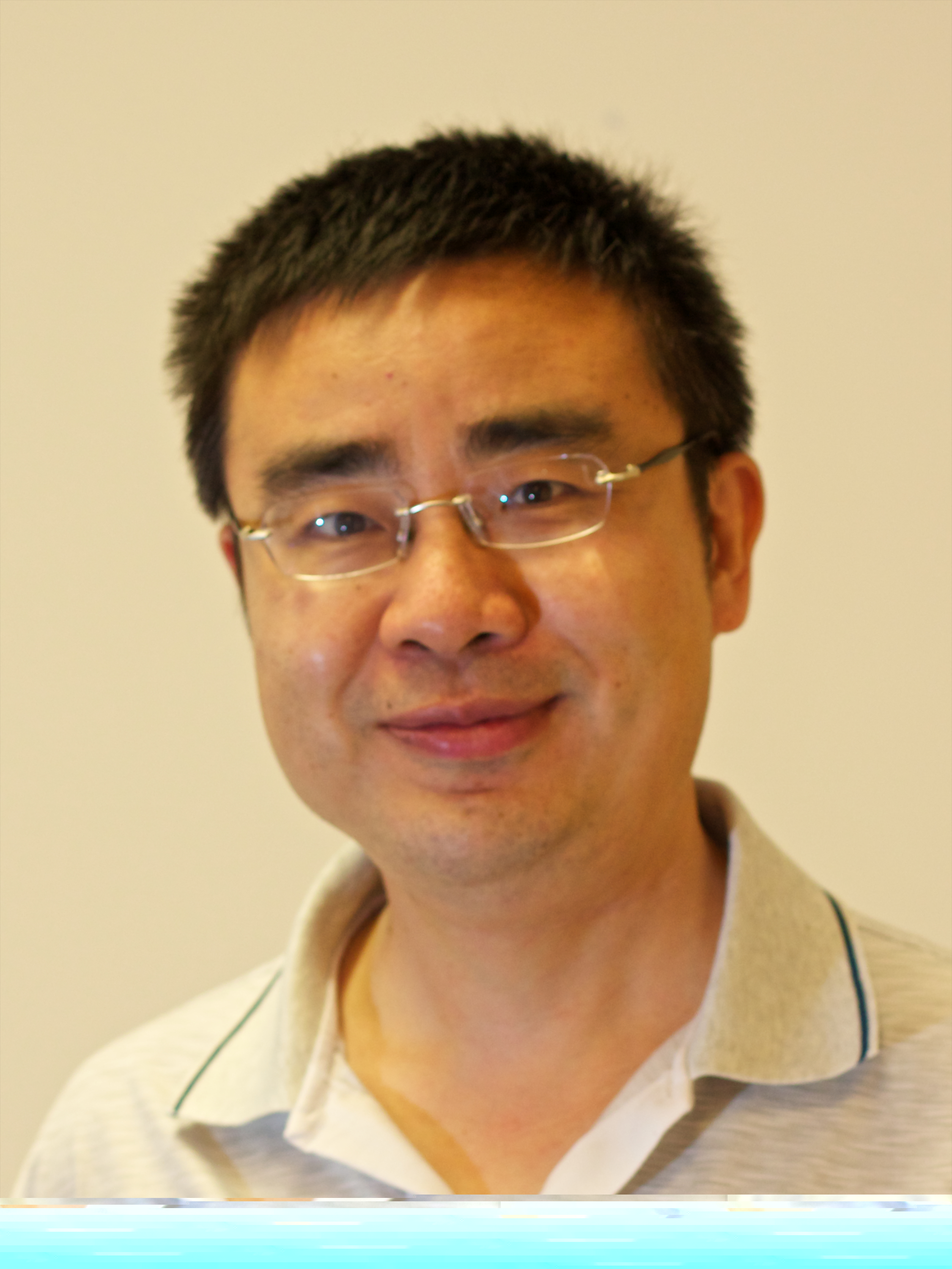}}]{Lian Wen (Larry)}
(Member, IEEE) is currently a Lecturer at the School of ICT at Griffith University. He earned a Bachelor’s degree in Mathematics from Peking University in 1987, followed by a Master’s degree in Electronic Engineering from the Chinese Academy of Space Technology in 1991. Subsequently, he worked as a Software Engineer and Project Manager across various IT companies before completing his Ph.D in Software Engineering at Griffith University in 2007. Larry’s research interests span four key areas: Software Engineering: Focused on Behaviour Engineering, Requirements Engineering, and Software Processes, Complex Systems and Scale-Free Networks, Logic Programming: With a particular emphasis on Answer Set Programming, Generative AI and Machine Learning.
\end{IEEEbiography}

\begin{IEEEbiography}[{\includegraphics[width=1in,height=1.25in,clip,keepaspectratio]{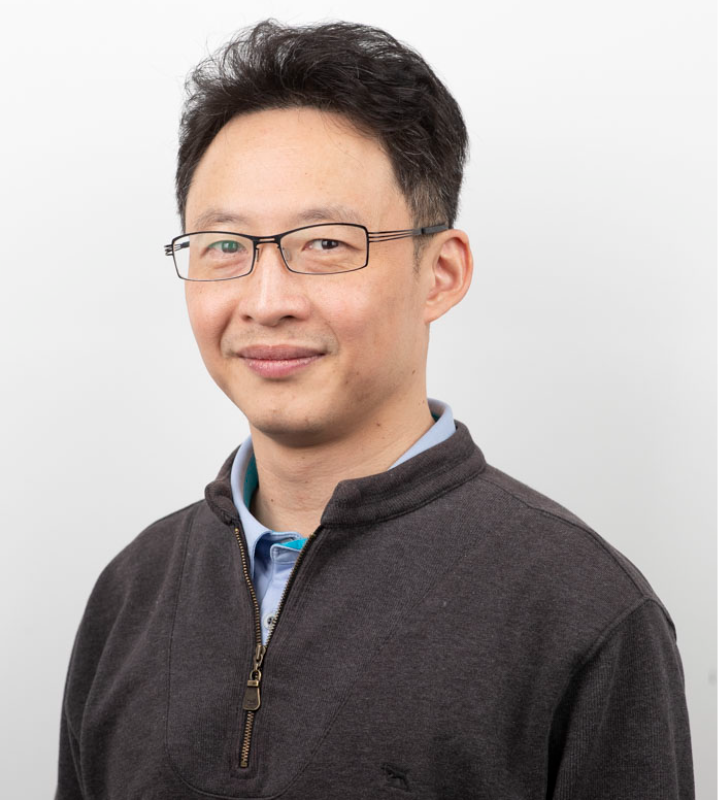}}]{David Chen}
 (Member, IEEE) obtained his Bachelor with first class Honours in 1995 and PhD in 2002 in Information Technology from Griffith University. He worked in the IT industry as a Technology Research Officer and a Software Engineer before returning to academia. He is currently a senior lecturer and serving as the Program Director for Bachelor of Information Technology in the School of Information and Communication Technology, Griffith University, Australia. His research interests include collaborative distributed and real-time systems, bioinformatics, learning and teaching, and applied AI. 
\end{IEEEbiography}

%% insert where needed to balance the two columns on the last page with
%% biographies
%%\newpage

%\begin{IEEEbiographynophoto}{Jane Doe}
%Biography text here.
%\end{IEEEbiographynophoto}
% ==== SWITCH OFF the BIO for submission
% ==== SWITCH OFF the BIO for submission

% You can push biographies down or up by placing
% a \vfill before or after them. The appropriate
% use of \vfill depends on what kind of text is
% on the last page and whether or not the columns
% are being equalized.
%\vfill

% Can be used to pull up biographies so that the bottom of the last one
% is flush with the other column.
%\enlargethispage{-5in}

% that's all folks
\end{document}